%% file: neurips_2026.tex
\documentclass{article}

 \usepackage[preprint]{neurips_2026}

\usepackage[utf8]{inputenc} 
\usepackage[T1]{fontenc}    
\usepackage{hyperref}       
\usepackage{url}    
\usepackage{makecell}
\usepackage{graphicx}
\usepackage{booktabs}       
\usepackage{amsfonts}       
\usepackage{nicefrac}       
\usepackage{microtype}      
\usepackage{xcolor}         
\usepackage{amsmath}
\usepackage{amssymb}
\usepackage{mathtools}
\usepackage{amsthm}
\usepackage{algorithm}
\usepackage{algpseudocode}
\newtheorem{theorem}{Theorem}[section]
\newtheorem{proposition}[theorem]{Proposition}
\newtheorem{lemma}[theorem]{Lemma}

\theoremstyle{definition}
\newtheorem{definition}[theorem]{Definition}

\theoremstyle{remark}

\newtheorem{fact}{Fact}

\newtheorem{assump}{Assumption}

\DeclareMathOperator*{\argmax}{argmax}

\newcommand{\expec}{\mathbb{E}}
\newcommand{\real}{\mathbb{R}}
\newcommand{\cA}{\mathcal{A}}

\newcommand{\cS}{\mathcal{S}}
\newcommand{\cF}{\mathcal{F}}
\newcommand{\cM}{\mathcal{M}}

\newcommand{\cP}{P}
\newcommand{\cR}{\mathcal{R}}

\newcommand{\T}{\mathbf{T}}
\newcommand{\cT}{\mathcal{T}}
\newcommand{\cX}{\mathcal{X}}

\newcommand{\NN}{\mathbb{N}}

\newcommand*{\infn}[1]{\left\|{#1}\right\|_{\infty}}

\renewcommand{\sp}[1]{\|{#1}\|_\text{\rm sp}}
\newcommand{\prob}{\mathbb{P}}

\DeclareMathOperator{\rint}{rint}

\usepackage{wrapfig}
\usepackage{todonotes}
\usepackage{caption}

\title{Learning Policy from a Single Trajectory in Average-Reward Markov Decision Process}

\author{Jongmin Lee\\
Seoul National University\\
\texttt{dlwhd2000@snu.ac.kr} \\
\And
Ernest K. Ryu \\
UCLA \\
\texttt{eryu@math.ucla.edu} 
\And
Vaneet Aggarwal \\
Purdue University \\
\texttt{vaneet@purdue.edu} 
}

\begin{document}

\maketitle

\begin{abstract}
While there is an extensive body of work characterizing the sample complexity of discounted cumulative-reward MDPs, finite sample analyses for average-reward MDPs have been limited, and most existing works rely on restrictive assumptions such as ergodicity or access to a generative model. In this work, we establish the first finite sample complexity guarantees from a single trajectory for weakly communicating average-reward MDPs. To this end, we study the dynamics of a single trajectory in weakly communicating MDPs  and based on this analysis, we develop novel model-free methods. Notably, our value-based and policy-based methods provide finite sample complexity guarantees of $\widetilde{O}(1/\varepsilon^2)$ and $\widetilde{O}(1/\varepsilon^4)$ from a single trajectory in weakly communicating MDPs, respectively. Furthermore, we introduce the first model-free method that requires no prior knowledge of problem-dependent quantities for communicating MDPs.
\end{abstract}

\section{Introduction}

Average-reward Markov decision processes (MDPs) provide a fundamental framework for modeling sequential decision-making problems whose objective is to maximize long-term steady-state performance. Under the average-reward criterion, characterizing the sample complexity for finding an $\varepsilon$-optimal policy has been an active topic in the reinforcement learning (RL) theory literature. However, most existing finite sample analyses rely either on restrictive sampling models, such as access to a generative model or i.i.d. offline data, or on restrictive structural assumptions, such as ergodicity. Overall, finite sample complexity analysis for average-reward RL from a single trajectory remains   unexplored.

\paragraph{Contribution.}
In this work, we establish finite sample complexity analyses from a single trajectory in weakly communicating average-reward MDPs for the first time (see Table~\ref{table:main}). To guarantee a PAC (probably approximately correct) bound for obtaining $\varepsilon$-optimal policy in this setup, we analyze the structure of weakly communicating MDPs through recurrent-transient classification and develop theoretical machinery based on probabilistic arguments with the stopping time and strong Markov property. Incorporating this machinery, our value-based method SAVIC and our policy-based method SCPMA achieve sample complexities of $\widetilde{O}(1/\varepsilon^2)$ and $\widetilde{O}(1/\varepsilon^4)$, respectively. We also introduce a first model-free method, SAVIC+ that requires no prior knowledge of problem-dependent quantities and achieves $\widetilde{O}(1/\varepsilon^2)$ sample complexity from a single trajectory for communicating MDPs, and develop a novel policy evaluation method for weakly communicating MDPs based on anchoring mechanism and variance-reduction techniques.

\begin{table}[t]
    \centering
    \resizebox{\columnwidth}{!}{%
    \begin{tabular}{ccccc}
        \toprule
        \textbf{Prior works} & \makecell{\textbf{MDP class}} & \makecell{\textbf{Complexity}} & \makecell{\textbf{Method type}} &
        \makecell{\textbf{Additional assumptions}}  \\
        \midrule
        Li et al.\ \cite{li2025stochastic} & ergodic & $\varepsilon^{-2}$ & model-free & -  \\
        \midrule
        Tuynman et al.\ \cite{tuynman2024finding} & communicating & $\varepsilon^{-2}$ & model-based & -  \\
        \midrule
        Lee et al.\ \citep{lee2025finite} & weakly communicating & $\varepsilon^{-8}$ & model-free &
        \makecell{finite coverage \&\\ ergodicity on behavior policy} \\
        \midrule
        Our work & weakly communicating & $\varepsilon^{-2}, \varepsilon^{-4}$ & model-free & -  \\
        \bottomrule
    \end{tabular}
    }
    \vspace{0.1in}
     \caption{Comparison of analyses for finding an $\varepsilon$-optimal policy from a single trajectory in average-reward MDPs. (Only the dependence on $\varepsilon$ is shown for the sample complexity.) For ergodic MDPs, \cite{li2025stochastic} establish finite sample complexity for stochastic first-order policy optimization methods that combine policy mirror descent with variance-reduced temporal-difference learning.  For communicating MDPs, \cite{tuynman2024finding} propose a model-based fixed-confidence algorithm without any prior knowledge of problem-dependent quantities and establish corresponding lower bounds. For weakly communicating MDPs, \cite{lee2025finite} study the sample complexity of average-reward fitted Q-iteration using function approximation and offline data generated from a single trajectory under additional finite coverage assumption and ergodicity assumption on the behavior policy. }
    \label{table:main}
\end{table}

\vspace{-0.05in}

\subsection{Prior works}
\paragraph{Average-reward MDPs.}
The average-reward MDP setup was first introduced in the dynamic programming literature by \cite{howard1960dynamic}, and \cite{blackwell1962discrete} established a theoretical framework for its analysis. In reinforcement learning (RL), average-reward MDPs have mainly been considered in the sample-based setting, where the transition matrix is unknown \citep{mahadevan1996average}. For this sampling setup, various methods have been proposed, including model-based methods \citep{jin2021towards, tuynman2024finding, zurek2024span}, value-based methods \citep{wan2024convergence, bravo2024stochastic, chen2025non}, and policy gradient methods \citep{bai2024regret,murthy2023convergence, kumar2024global}. On the theoretical side, sample complexity analyses under generative model access \citep{wang2017primal, zhang2023sharper, li2025stochastic, jin2024feasible, lee2025near} and offline datasets \citep{ozdaglar2024, gabbianelli2024offline, zurek2025optimal, lee2025finite}, as well as regret minimization frameworks \citep{burnetas1997optimal, Jaksch2010, zhang2019regret, boone2024achieving}, have been actively studied.

\paragraph{Value Iteration}
Value iteration (VI) was first introduced in the dynamic programming literature \cite{bellman1957Markovian} and serves as a fundamental algorithm based on computing the value functions. The sample-based variants, such as TD-Learning~\cite{Sutton1988}, Fitted Value Iteration~\cite{ernst2005tree,munos2008finite}, and Deep Q-Network~\cite{MnihKavukcuogluSilveretal2015} are the workhorses of modern reinforcement learning algorithms~\cite{Bertsekas96,sutton2018reinforcement,SzepesvariBook10}, and routinely applied in diverse settings, including factored MDPs \cite{rosenberg2021oracle}, robust MDPs \cite{kumarefficient}, MDPs with reward machines \cite{bourel2023exploration}, MDPs with options \cite{fruit2017regret}, and generative model  \cite{wainwright2019variance, sidford2023variance, lee2025near}. For the average-reward MDP, the convergence of VI in average-reward MDPs also has been extensively studied by \citep{seneta2006non, hubner1977improved,federgruen1978contraction, van1981stochastic, della2012illustrated, schweitzer1977asymptotic, schweitzer1979geometric}.  

\paragraph{Policy gradient methods.}
Policy gradient methods \citep{williams1992simple,sutton1999,konda1999actor,kakade2001natural} are foundational reinforcement learning algorithms, commonly implemented with deep neural networks for policy parameterization \citep{schulman2015trust,schulman2017proximal}. In line with their practical success, convergence and sample complexity of policy gradient variants have been extensively studied across settings \citep{shani2020adaptive, mei2020global, agarwal2021theory, cen2022fast, xiao2022convergence, bhandari2024global,mondal2024improved}. For the average-reward MDP, \cite{sutton1999, marbach2001simulation, baxter2000direct} establish policy gradient theorem in unichain MDP, and recently, \cite{lee2025multi} establishes policy gradient theorem for weakly communicating and multichain MDP. The convergence of variants of policy gradient methods including natural policy gradient, projected policy gradient, and policy mirror ascent have been studied by  \cite{murthy2023convergence,kumar2024global,lee2025multi}, as well as in parameterized settings \cite{bai2024regret, patel2024towards, ganesh2024order, ganeshsharper, ganesh2025regret}.

\subsection{Preliminaries and notations}

\paragraph{Average reward.}
Let $\cM(\cX)$ be the space of probability distributions over a set $\cX$. We write $(\cS, \cA, P, r)$ to denote the infinite-horizon undiscounted MDP with finite state space $\cS$, finite action space $\cA$, transition matrix $P\colon \cS \times \cA \rightarrow \cM(\cS)$, and bounded reward $r\colon  \cS \times \cA \rightarrow [-B,B]$. Denote $\pi\colon \cS \rightarrow \cM(\cA)$ for a policy, 
\begin{align*}
    J^{\pi}(s,a)=\lim_{H\rightarrow \infty} \frac{1}{H}\expec_{\pi}\bigg[\sum^{H-1}_{h=0}  r(s_h, a_h) \,\Big|\, s_0=s, a_0=a\bigg]
\end{align*}
for average-rewards of a given policy \cite[Section 9.1]{Puterman2014}, and \begin{align*}
Q^\pi(s,a)=\lim _{H\rightarrow \infty} \frac{1}{H}\sum_{h=1}^{H} \expec_\pi \bigg[\sum^{h-1}_{i=0}\big(r(s_i,a_i)-J^{\pi}(s_i,a_i)\big)\,\Big|\, s_0=s, a_0=a\bigg]    
\end{align*}
for state-action value function (under an aperiodicity assumption, averaging with respect to $H$ can be omitted in the definition \cite[Section 8.2]{Puterman2014}), where $\expec_{\pi}$ denotes the expected value over all trajectories $(s_0, a_0, s_1, a_1, \dots, s_{H-1}, a_{H-1})$ induced by $P$ and $\pi$. Denote $r^{\pi}(s)=\mathbb{E}_{a \sim \pi(\cdot\,|\,s) }\left[r(s,a)\right]$ for the reward induced by policy $\pi$. For $x\in\real^d$ we denote by $\|x\|_{\infty}=\max_{i}|x_i|$ its infinity norm and by $\sp{x}=\max_{i}x_i-\min_{i}x_i$ its  span seminorm.  We say $J^{\star}$ is optimal average reward if $J^{\star}=\max_{\pi}J^{\pi}$ and optimal average reward always exists in average-reward MDP \citep[Section 9.1]{Puterman2014}. We say $\pi$ is an $\varepsilon$-optimal policy if $\infn{J^{\star}-J^{\pi}} \le \varepsilon$. We denote  $P_{\star}= 
\lim_{H\rightarrow \infty}\frac{1}{H} \sum^{H-1}_{i=0}P^i$
for the Ces\`aro limit of a stochastic matrix $P$. (Ces\`aro limit of a stochastic matrix always exists \cite[Theorem A.6]{Puterman2014}). We denote the transition matrix induced by policy $\pi$ and $P$ as
\[
P^{\pi}((s,a), (s',a'))=
\mathrm{Prob}((s,a)\rightarrow (s',a')\,|\,
s'\sim P(\cdot\,|\,s,a),a' \sim \pi(\cdot\,|\,s')),
\]
where $P^{\pi} \in \mathbb{R}^{|\cS||\cA| \times |\cS||\cA|}$.
 Then, by definition, we can write  \cite[Section 8.2]{Puterman2014}
\[J^\pi=P^\pi_{\star}r, \qquad Q^\pi=(I-P^\pi+P^{\pi}_\star)^{-1}(I-P_\star^\pi)r
\]
For every policy $\pi$, it is known that $Q^\pi$ and $ J^\pi$ satisfy the following \emph{Bellman equations} \citep[Theorem~8.2.6] {Puterman2014}:    
\begin{align*}
  &P^\pi J^\pi=J^\pi , \qquad r+ P^\pi Q^\pi = J^\pi+Q^\pi.
\end{align*} 

\paragraph{MDP classes.}
 MDPs are classified according to the structure of their transition matrices. 
(For definitions of irreducible classes, accessibility, and aperiodicity of transition matrices, refer to \cite[Appendix A.2]{Puterman2014}.)
%
%
An MDP is \emph{ergodic} if, for every policy $\pi$, the induced transition matrix has a single recurrent class and is aperiodic. An MDP is \emph{communicating} if there exists a set of states such that every state in the set is accessible from every other state in the set under some policy. An MDP is \emph{weakly communicating} if there is a set of states where each state in the set is accessible from every other state in the set under some policy, together with a possibly empty set of states that are transient under all policies. Every ergodic MDP is communicating, and every communicating MDP is weakly communicating. 
  
  
\paragraph{Value Iteration.}
Given an undiscounted MDP $(\cS, \cA, P, r)$, for $Q \in \real^{|\cS \times \cA|}$, define the Bellman optimality and consistency operators $\T$ and  $\T^\pi$ as
\begin{align*}
\T Q(s,a)&=r(s,a)+\mathbb{E}_{s'\sim P(\cdot\,|\,s,a)}\Big[\max_{a' \in \cA} Q(s',a')\Big]
\\\T^\pi Q(s,a)&=r(s,a)+\mathbb{E}_{s'\sim P(\cdot\,|\,s,a), a' \sim \pi(\cdot \,|\, s')}[ Q(s',a')]
\end{align*}
for all $s \in \cS$ and $a \in \cA$, respectively.
We define the standard Value Iteration (VI) for Bellman consistency and optimality operators as
\[ Q^{k}=\T Q^{k-1}, \qquad  Q^{k}=\T^\pi Q^{k-1} \qquad\text{ for } k=1,2,\dots,K,
\]
 where $Q^0$ is an initial point.   For notational conciseness, we write $\T^{\pi}Q=r+P^{\pi}Q $.
and $ \mathbb{E}_{ a \sim \pi(\cdot \,|\, s)}[ Q(s,a)]=Q (s, \pi(s))$. 

\paragraph{Policy mirror ascent.}
Given a weakly communicating average-reward MDP $(\cS, \cA, P, r)$, for a positive policy $\pi\in \Pi_+$, the policy gradient $\nabla J^\pi$ is
\[\nabla_{\theta}  J^{\pi_\theta} =    \sum_{s \in \cR}
\sum_{a \in \cA}
d^{\pi_\theta}(s) \nabla_{\theta} \pi_{\theta} (a \mid s) Q^{\pi_\theta}(s,a), \]
where $\cR$ is a single recurrent class and $d^{\pi}$ is the unique stationary distribution \cite{lee2025multi}.
Next, consider the direct parameterization $\pi_{\theta} (a \,|\, s) = \theta_{s,a},$ where $\theta \in \real^{|\cS|\times |\cA|}$ satisfies $\sum_{a \in \cA} \theta_{s,a}=1$ and $\theta_{s,a}\ge 0$ for all $s\in \cS$ and $a\in \cA$. 
Let \(h:\cM(\cA)\to\mathbb{R}\) be a strictly convex function and continuously
differentiable on the relative interior of \(\cM(\cA)\), denoted as
\(\rint\cM(\cA)\).
Define the Bregman divergence generated by \(h\) as
\[
D(p,p') \;=\; h(p)-h(p')-\langle \nabla h(p'),\, p-p' \rangle,
\qquad\forall\,p\in\cM(\cA),\,p'\in\rint\,\cM(\cA).
\]
Then, we define policy mirror ascent as 
\[\pi_{k+1}(\cdot \,|\, s)=
 \argmax_{p\in \cM (\cA)}
 \left\{
 b_k \sum_{a \in \cA}Q^{\pi_k}(s,a)p(a)
 - D\big(p(\cdot),\pi_{k}(\cdot \,|\, s)\big)\right\}\]
for all $s\in \cR$ where $\pi_0$ is an initial policy, and $b_k$ is step size.
 
\section{A Single Trajectory in Weakly Communicating MDPs}\label{sec:sin}

In this section, we provide stochastic analysis for a single trajectory in weakly communicating MDPs. We first study the structure of weakly communicating MDPs through recurrent-transient classification. For the analysis, we define the class of policies
\begin{alignat*}{3}
\Pi = \text{set of all policies}=\cM(\cA)^\cS, \qquad \Pi_{+} = \{ \pi\in \Pi \,|\, \pi (a \,|\, s) >0 \text{ for all } (s,a) \in \cS\times \cA\},
\end{alignat*}
so that $\Pi_+$ is the (relative) interior of $\Pi$. 
\subsection{Recurrent-transient classifications}
\begin{definition}
Given a policy $\pi\in \Pi$, a state $s\in \cS$ is recurrent if its return time starting from $s$ is finite with probability $1$. Otherwise, $s$ is transient.
\end{definition}
With the policy class $\Pi_+$ instead of $\Pi$, \citep{lee2025policy} shows that  the recurrent--transient classification of states is invariant across all $\pi \in \Pi_+$ on finite state and action spaces. 
Together with the structure of weakly communicating MDPs, the following fact follows. 
\begin{fact}[Fact 1, Lemma 26, \citep{lee2026policy}]\label{fact:rec}
In a weakly communicating MDP, 
$\cS = \cR \cup \cT$, where $\cR$ is a single recurrent class and $\cT$ is a transient set, and this decomposition is invariant for all policy $\pi \in \Pi_+$. Moreover, $\cT$ continues to be transient set  for $\pi \in \Pi$. 
\end{fact}

Fact~\ref{fact:rec} implies that, for $\pi\in \Pi_+$, $P^\pi$ has only a single recurrent class for a weakly communicating MDP, and the recurrent-transient classification is determined by the transition kernel $P$, not by the particular choice of $\pi\in \Pi_+$. In contrast, a policy $\pi\in \Pi\setminus\Pi_+$, i.e., a policy that assigns zero probability to some actions, may induce a different classification and a multichain structure. But $\cT$ continues to be transient set for a policy $\pi\in \Pi\setminus\Pi_+$ for weakly communicating MDP.

\begin{wrapfigure}{r}{0.42\textwidth}
\vspace{-.15in}
\captionsetup{type=algorithm}
\hrule
\vspace{2pt}
\begin{minipage}{0.42\textwidth}
\caption{\mbox{\sc recur}$(n_1,n_2)$}
\label{alg:rec}
\vspace{-.1in}
\hrule
\begin{algorithmic}
   \State {\bfseries Input:} 
   $n_1, n_2 \in \mathbb{N}$
      \State {\bfseries Initialize:} 
   $R =\{\}$
       \For{$t=0,1,\dots, n_{1}+n_{2}$}
    \State $a_t \sim \pi_b(\cdot \,|\, s_{t}), s_{t+1} \sim P(\cdot \,|\, s_t,a_t)$
    \If{$t > n_{1}$}
    \State $R \gets R \cup \{s_{t+1}\}$
    \EndIf
    \EndFor
    \If{$R=\{\}$}
      \State $R \gets \cS$
    \EndIf
 \State {\bfseries Output:} $R$
\end{algorithmic}
\hrule
\end{minipage}
\vspace{-.2in}
\end{wrapfigure}

\subsection{Escaping from transient states and listing the recurrent states}

Consider a behavior policy $\pi_b \in \Pi_+$ and a single trajectory $ \{s_t, a_t,r_t\}^\infty_{t=0}$, where $a_t \sim \pi_b(\cdot \,|\, s_t), r_t=r(s_t,a_t)$, and $s_{t+1} \sim P(\cdot \,|\, s_t,a_t)$. 

If the starting state $s_0$ is transient, the trajectory will reach the recurrent states in finite time by the definition of transient states. To specify this time, we define the hitting time and expected hitting time as 
\[
\tau_{\mathrm{hit}} = \inf_{t \ge 0} \{s_t \in \cR\}, 
\qquad 
t_{\mathrm{hit}}= \max_{s\in \cS} \expec_{\pi_b}[\tau_{\mathrm{hit}} \,|\, s_0=s ].
\] 
Then, the following lemma holds by the Markov's inequality and Markov property.
\begin{lemma}\label{lem:hit}
 Consider a single trajectory generated from a weakly communicating MDP under a behavior policy $\pi_b \in \Pi_+$. 
Let  $ \delta \in (0,1)$ and $L \ge  \left\lceil e t_{\mathrm{hit}} \right\rceil\left\lceil\ln\!\left(\frac{1}{\delta}\right) \right\rceil$. With probability at least $1-\delta$, $s_L \in \mathcal{R}$.
\end{lemma}

After hitting the recurrent class, the trajectory will remain in the single recurrent class. In a single recurrent class, every recurrent state can be reached in a finite time, and to specify the time for reaching all recurrent states, we define the covering time and expected covering time as
\[
\tau_{\mathrm{cov}}=\inf_{t\ge 0}\{ \cup^{t-1}_{j=0}\{s_j\}=\mathcal{R} \},
\qquad
t_{\mathrm{cov}}=\max_{s\in \cR} \expec [\tau_{\mathrm{cov}} \,|\, s_0=s].
\]
Again, by the Markov's inequality and Markov property, the following fact holds.
\begin{fact}[\cite
{chen2022learning,lee2025multi}]\label{fact:cov}
Consider a single trajectory generated by a Markov process with a single recurrent class. Let $ \delta \in (0,1)
$ and $L \ge \left\lceil e t_{\mathrm{cov}}\right\rceil  \left\lceil\ln \!\left(\frac{1}{\delta}\right)\right\rceil$. With probability at least $1-\delta$, $
\bigcup^{L-1}_{j=0}\{s_j\} = \mathcal{R}.
$
\end{fact}

Based on Lemma~\ref{lem:hit} and Fact~\ref{fact:cov}, Algorithm~\ref{alg:rec} generates a list of all recurrent states with high probability under appropriate choice of parameters. As the analysis in the next section shows, the cardinality of the recurrent class gives refined sample complexity. We also note that the expected hitting time and covering time are finite when the state space is finite \citep{lee2025policy, levin2017Markov}.

\subsection{i.i.d. samples from a single trajectory}

The main technical challenge with single trajectory samples is handling the dependency between samples and related random variables. For a single recurrent class, we can obtain i.i.d. samples from a single trajectory by using stopping times and the strong Markov property as the following fact shows.

\begin{fact}\label{fact:iid}
 Consider a single trajectory generated by a Markov process with a single recurrent class. For a given state $s\in \cS$, define the stopping time $\tau_i = \inf \{k > \tau_{i-1} : s_{k}=s\}$
with $\tau_0=-1$. Then, for $M \in \mathbb{N}$, $\{s_{\tau_i+1}\}^M_{i=1}$ are i.i.d. and $\tau_M < \infty$ with probability $1$.
\end{fact}

Building on Fact~\ref{fact:iid}, we introduce our key inequality used in our analysis. To this end, we define $d_{\mathrm{min}}$ as $\min_{s\in \cR } d(s)$, where $d$ is the unique stationary distribution of the given transition probability.

\begin{lemma}\label{lem:sam_bd}
    Suppose $\{s_i\}^\infty_{i=0}$ is a trajectory generated by a Markov process with a single recurrent class. For each state $s\in \cS$ and $N,M \in \mathbb{N}$, define the stopping time $\tau_i = \inf \{k > \tau_{i-1} \,|\, s_{k}=s\}$
with $\tau_0=-1$, and random variable $I_s= \min \{M, \max \{k \,|\,  \tau_k \le N\}\}$. Then, for $\lambda \in \real$ and $f : \cS \rightarrow [a,b]$, we have 
\begin{align*}
    \expec\left[\exp \left(\frac{\lambda}{M}\sum^{I_s}_{i=1}(f(s_{\tau_{i}+1}) -\expec [f(s_{\tau_{i}+1})]\right)\right]  \le \mathrm{exp}\left( \frac{ \lambda^2(a-b)^2}{8M}\right).
\end{align*}
Moreover, if $ N \ge \lceil e^2( \ln|\cS|+e)( t_{\mathrm{cov}}+(M-1)/d_{\mathrm{min}})\rceil \lceil\ln \left(\frac{1}{\delta }\right)\rceil$, with probability at least $1-\delta$,  $I_s = M$  $\forall s \in \cS$.
\end{lemma}

\begin{wrapfigure}{r}{0.46\textwidth}
\vspace{-.14in}
\captionsetup{type=algorithm}
\hrule
\vspace{2pt}
\begin{minipage}{0.46\textwidth}
 \caption{\mbox{\sc sample}$( n,  m, d, R)$}
   \label{alg:sam_1}
\vspace{-.1in}
\hrule

\begin{algorithmic}
   \State {\bfseries Input:} 
  \ $d \in \real^{R}$, $n, m \in \mathbb{N}, R \subset \cS$
   \State {\bfseries Initialize:}    
   $m_{s,a}, D_{s,a} =0 \, \forall (s,a) \in R \times \cA $
   \For{$t =0,1,\dots, n$}
     \State $a_{t} \sim \pi_b(\cdot \,|\, s_{t}), s_{t+1} \sim P(\cdot \,|\, s_t,a_t)$
        \If{$s_t,a_t \in R\times \cA, m_{s_t,a_t} <m$}
   \State $D_{s_t,a_t} \gets d(s_{t+1}) + D_{s_t,a_t}$
     \State $m_{s_t,a_t} \gets m_{s_t,a_t}+1$
   \EndIf
  \EndFor
  \For{$(s,a) \in R \times \cA$}
  \State $D_{s,a} \gets \frac{1}{m}D_{s,a} $
  \EndFor
 \State {\bfseries Output:} $D$
\end{algorithmic}
\hrule
\end{minipage}
\end{wrapfigure}
Based on Lemma~\ref{lem:sam_bd} and the dynamics of a single trajectory in weakly communicating MDPs, we will show that Algorithm~\ref{alg:sam_1} successfully approximates the expectation of given quantities with respect to transition probability for a given $R \subset \cS$. We note that while prior analyses for average-reward MDPs \cite{wei2020model,bai2024regret} rely on mixing assumptions to  address the dependency of random variables related to algorithms, we utilize insights into the structure of weakly communicating MDPs based on recurrent-transient classifications, the strong Markov property, and stopping time arguments to resolve this technical issue.

In the next section, we will combine this machinery with model-free algorithms and establish finite sample complexity from a single trajectory in weakly communicating average-reward MDPs.

\section{Stochastic Anchored Value Iteration}\label{sec:savi}
In this section, we present a novel value-based method that provide finite sample complexity guarantees from a single trajectory for weakly communicating MDPs, and a prior-knowledge-free variant for communicating MDPs.
\subsection{Framework overview}
Our novel value-based method, SAVIC, is divided into three stages. In the first stage, as discussed in the previous section, the agent escapes from transient states and collects the full list of recurrent states through Algorithm~\ref{alg:rec}. Based on $R$ obtained from RECUR, we focus on $R$ instead of $\cS$, and we initialize parameters as described in Algorithm~\ref{alg:SAVIC}, where $r_{R\times \cA} \in \real^{R\times \cA} $ is defined by $r_R(s,a)=r(s,a)$  $\forall (s,a)\in R \times \cA$ and $\mathbf{0}_R$ and $\mathbf{0}_{R\times \cA}$ denote $|R|$- and $|R\times\cA|$- dimensional zero vectors, respectively. Note in previous section, we show that with high probability, $R=\cR$.  

In the next stage,  during the first for-loop of Algorithm~\ref{alg:SAVIC}, we utilize an anchoring mechanism and a variance-reduction technique to obtain a near-optimal policy for the recurrent class with efficient sample complexity. The anchoring mechanism, also classically known as the Halpern iteration \cite{halpern1967fixed}, has been widely studied in minimax optimization and fixed-point problems \cite{sabach2017first, Lieder2021halpern,yoon2021accelerated,park2022exact, bravo2026minimax}. Recently, Anchored Value Iteration  has been studied to obtain finite-time bounds on policy error for average-reward MDPs \cite{bravostochastic,lee2025multi,lee2025near,lee2025finite, zurek2025faster}, and in particular, \emph{Anchored Q-Value Iteration} is
\[
Q^{k} = (1-\beta_k) Q^0+\beta_k \T Q^{k-1} \qquad\text{for } k=1,2,\dots. \tag{Anc-QI}
\]

where the parameter $\beta_k$ is chosen appropriately. Compared to standard Value Iteration, Anc-QI obtains the next iterate as a convex combination of the output of $\T$ and the starting point $Q^0$. We call the $(1-\beta_k)Q^0$ term the \emph{anchor term}, since it pulls the iterates back toward the starting point $Q^0$. With this Anc-QI, \cite{sabach2017first, lee2025multi} establish non-asymptotic convergence in the average-reward setup. In weakly communicating MDPs, Anc-QI with $\beta_k=k/(k+2)$ exhibits the following convergence rate:
\[
 \infn{J^\star-J^{\pi_k}} \le \sp{Q^k-\T(Q^k)}
 \leq \mbox{$\frac{4}{k+1}$}\sp{Q^0-Q^\star},
\]
where $Q^{\star}$ is the optimal state-action value function, and $\pi_k$ is the greedy policy induced by $\T(Q^k)$. 
\begin{wrapfigure}{r}{0.6\textwidth}
 \vspace{-.17in}
\captionsetup{type=algorithm}
\hrule
\vspace{2pt}
\begin{minipage}{0.6\textwidth}
\caption{\sc \mbox{SAVIC}\,$( n_{\mathrm{esc}},n_{\mathrm{rec}}, n, \{t_k\}^n_{k=0}, \varepsilon, \delta)$\\w.c. = weakly communicating, c.= communicating}
   \label{alg:SAVIC}
\vspace{-.1in}
\hrule
\begin{algorithmic}
     \State {\bfseries Input:} $n_{\mathrm{esc}},n_{\mathrm{rec}}, n\in\NN$\,,\,$\varepsilon>0$\,,\,$\delta\in(0,1)$ 
     \State {\bfseries Initialize:} $\beta_k=k/(k\!+\!2), $  $c_k=5(k\!+\!2)\ln^2(k\!+\!2) \forall \, k \in \NN$
\State $R=\mbox{\sc recur}(n_{\mathrm{esc}}, n_{\mathrm{rec}})$
\State{$\eta = \ln(8|R||\cA|(n\!+\!1)/\delta)$}
 \State $Q^0=\mathbf{0}_{R\times \cA}, T^{-1}=r_{R\times \cA}, h^{-1}=\mathbf{0}_R$   
    \For{ $k=0,\ldots,n$ }
     \State $Q^{k} =(1\!-\!\beta_k)\,Q^0+\beta_k\, T^{k-1}$
     \State $h^k(\cdot)= \max_{a\in \cA} Q^{k}(\cdot, a)$
     \State $d^k=h^k-h^{k-1}$
    \State $m_k= \max\{\lceil \eta\,c_k\sp{d^k}^2/\varepsilon^{2}\rceil,1\}$
    \State{$D^k =
\begin{cases}
\mbox{\sc SAMPLE}(t_k,  m_k, d^k, R) & \text{w.c.} \\
\mbox{\sc SAMPLE+}( m_k, d^k) & \text{c.}
\end{cases}$}
     \State $T^k=T^{k-1}+D^k$ 
   \EndFor
       \State $\pi^{n}(s)
\in \argmax_{a\in \cA} Q^{n}(s,a) \quad(\forall s\in R)$ 
        \State{$C_{\mathrm{pur}} =
\begin{cases}
 1/((|\cA|+1)(\infn{Q^n}+1)) & \text{w.c.} \\
$0$ & \text{c.}
\end{cases}$}
       \State{$\pi^{n}_{ \varepsilon}(a \,|\,s) =
\begin{cases}
 \frac{\varepsilon  C_{\mathrm{pur}}+\pi^n(a\,|\, s)}{1+|\cA|\varepsilon C_{\mathrm{pur}}} \quad(\forall s\in R)
 \\
1/|\cA| \qquad\qquad\,\,(\forall s\in R^\complement)
\end{cases}$}
   \State {\bfseries Output:} $(Q^{n} ,T^n, \pi_{\varepsilon}^{n})$ 
\end{algorithmic}
\hrule
\end{minipage}
\vspace{-.17in}
\end{wrapfigure}
Adding to this anchoring mechanism,  we utilize a \emph{recursive sampling} technique borrowed from \cite{jin2024truncated, lee2025near} to obtain the efficient sample complexity. The basic idea is to exploit the previous approximation $T^{k-1}\approx \T(Q^{k-1})$ to approximate $\T(Q^{k})$ by estimating the difference $\T(Q^{k})-\T(Q^{k-1})=\cP d^k$, where $d^k=h^k-h^{k-1}$, and adding it to $T^{k-1}$. Then, denoting by $D^k\approx\cP d^k$ the matrix sampled at the $k$-th step, we have 
\begin{equation*}\label{Eq:telescope}
    \mbox{$T^k=r+\sum_{i=0}^k D^i$}
\end{equation*}
so that all previous estimates $\{D^i\}^k_{i=0}$ are used to approximate $\T(Q^{k})$.
 
In the last stage,  during the second for-loop of Algorithm~\ref{alg:SAVIC}, we perturb the policy to obtain a complete $\varepsilon$-optimal policy for the whole state space. Note that in the second stage, we only utilized samples from the single recurrent class. In other words, there is no guarantee that a near-optimal policy on the recurrent class is also near-optimal on the whole state space. To address this issue,   we introduce the following lemma proved by utilizing the structure of weakly communicating MDPs and the fact that our Anc-QI simultaneously reduces Bellman error and policy error.  
\begin{lemma}\label{lem:pur}
Consider a weakly communicating MDP. Let $Q\in\real^{\cR\times\cA}$ and $\pi$ be a policy satisfying $\pi(s)\in\argmax_{a\in\cA}~Q(s,a) \, \forall s \in \cR$.
For $\epsilon>0$, if policy $\pi_\epsilon \in \Pi_+$ satisfies $\max_{s \in \cR}\|\pi(\cdot \,|\, s)-\pi_\epsilon (\cdot \,|\, s)\|_1\le \epsilon$, we have
$    \infn{J^\star-J^{\pi_\epsilon}} 
    \le \sp{\T Q -Q}
    +\epsilon  \infn{Q}.$
\end{lemma}
Intuitively speaking, this lemma ensures that a smaller Bellman error on the recurrent class induces a small policy error over all states if we properly choose the perturbation in weakly communicating MDP. Note that $J^\star, J^{\pi_\epsilon} \in \real^{\cS \times \cA}$.

Through these three stages, as shown in the next subsection, SAVIC provides an $\varepsilon$-optimal policy with efficient sample complexity from a single trajectory in weakly communicating average-reward MDPs. 
We lastly note that the framework of the second stage basically follows \cite{lee2025near} which analyze sample complexity under the generative model.

\subsection{Finite sample complexity for weakly communicating MDPs}

To establish the sample complexity of SAVIC, with high probability, we first ensure that our sampling process provides a good approximation $T^{k}\approx \T(Q^{k})$ for all $k$.  To this end, we define $d^{\pi_b}_{\mathrm{min}}$ as $\min_{s,a\in \cR\times \cA } d^{\pi_b}(s,a)$, where $d^{\pi_b}$ is the unique stationary distribution of $P^{\pi_b}$, and define $t_{\mathrm{cov}}$ and $ t_{\mathrm{hit}}$ with respect to $P^{\pi_b}$ and its recurrent set, $\cR\times \cA$. We note that lemmas and facts in Section~\ref{sec:sin} continue to hold with $P^{\pi_b}$ as addressed in Appendix~\ref{ap:sa}.  

\begin{proposition}\label{prop:savic}
Consider a weakly communicating MDP. Let $n_{\mathrm{esc}} \ge\lceil e t_{\mathrm{hit}}\rceil \lceil\ln (\frac{4}{\delta})\rceil$, $n_{\mathrm{rec}}  \ge  \lceil  e t_{\mathrm{cov}}\rceil \lceil \ln (\frac{4}{\delta})\rceil$, $
t_k \ge  \left\lceil e^2( \ln(|R||\cA|)+e)\left( t_{\mathrm{cov}}+\frac{m_k-1}{d^{\pi_b}_{\mathrm{min}}}\right)
\right\rceil \lceil \ln(\frac{4(n+1)}{\delta})\rceil$,
and let $T^k, Q^k$ be the iterates generated by $
\mbox{\sc \mbox{SAVIC}}(n_{\mathrm{esc}},n_{\mathrm{rec}},n, \{t_k\}^n_{k=0}, \varepsilon, \delta).$
Then, with probability at least $1-\delta$, we have $
\infn{T^k-\T(Q^k)}\le  \varepsilon $ for all $k=0,\ldots,n$.
\end{proposition}

We defer the proof to the Appendix~\ref{omit:sec3}, but briefly mention that the proof is an adaptation of the Azuma-Hoeffding inequality combined with the sampling arguments in Section~\ref{sec:sin} and $t_k$ is chosen so that trajectories are rolled out until every recurrent state is visited $m_k$ times with high probability.

We now present the sample complexity bound for \mbox{\rm SAVIC}.
\begin{theorem}\label{thm:savic}
Consider a weakly communicating MDP. Let  $
n_{\mathrm{esc}}=\lceil e t_{\mathrm{hit}}\rceil \lceil\ln (\frac{4}{\delta})\rceil$, $n_{\mathrm{rec}} =  \lceil  e t_{\mathrm{cov}}\rceil \lceil \ln (\frac{4}{\delta})\rceil$, $n= \lceil 32\sp{Q^\star}/\varepsilon\rceil$, and $
t_k =  \left\lceil e^2( \ln(|R||\cA|)+e)\left( t_{\mathrm{cov}}+\frac{m_k-1}{d^{\pi_b}_{\mathrm{min}}}\right)
\right\rceil   \lceil \ln(\frac{4(n+1)}{\delta})\rceil $.
Let $(Q^n,T^n,\pi^n_{ \varepsilon})$ be the output computed by $
\sc \mbox{SAVIC}( n_{\mathrm{esc}},n_{\mathrm{rec}}, n, \{t_k\}^n_{k=0}, \varepsilon/16, \delta)$ with $\varepsilon \le 1 \le \sp{Q^\star}$.
Then, with probability at least $1-\delta$, we have $
\infn{J^\star-J^{\pi^n_{ \varepsilon}}} \le \varepsilon$
with sample complexity
\[
\widetilde{O}\left(t_{\mathrm{hit}}+t_{\mathrm{cov}}\sp{Q^\star}/\varepsilon
+ 1/d^{\pi_b}_{\mathrm{min}} \sp{Q^\star}^2/\varepsilon^{2}\right),
\]
where $\widetilde{O}(\cdot)$ hides logarithmic factors including $\log(1/\delta)$.
\end{theorem}

The proof exploits Proposition~\ref{prop:savic} and anchoring mechanism to establish an upper bound of the Bellman residual
error on recurrent class with efficient sample complexity. We also  utilize Lemma~\ref{lem:pur} to obtain complete $\epsilon$-optimal policy for all states. To the best of our knowledge, this is the first finite sample complexity result from a single trajectory in weakly communicating MDPs.

\subsection{Without prior knowledge for communicating MDPs}
\label{others}

\begin{wrapfigure}{r}{0.44\textwidth}
\vspace{-.14in}
\captionsetup{type=algorithm}
\hrule
\vspace{2pt}
\begin{minipage}{0.44\textwidth}
 \caption{\mbox{\sc SAMPLE+}$(d, m)$}
   \label{alg:sam_2}
\vspace{-.1in}
\hrule
\begin{algorithmic}
   \State {\bfseries Input:} 
   $d \in \real^{\cS}$\,;\, $m \in \mathbb{N}$
      \State {\bfseries Initialize:}    
   $m_{s,a}, D_{s,a} \!=\!0,\! \forall (s,a) \!\in\! \cS \!\times \!\cA $
   \State {\bfseries for} $t=0,1,\dots$ \,\,\textbf{do}
     \State \quad\,\,\, $a_{t} \sim \pi_b(\cdot \,|\, s_{t}), s_{t+1} \sim P(\cdot \,|\, s_t,a_t)$
     \State {\quad\,\,\,  \bfseries If} $m_{s_t,a_t} <m$ \,\,\textbf{then} 
   \State  \quad\,\,\qquad $D_{s_t,a_t}\! \gets \! \frac{d(s_{t+1})}{m_{s_t,a_t}+1} + \frac{m_{s_t,a_t}D_{s_t,a_t}}{m_{s_t,a_t}+1}$
     \State {\quad\,\,\,\,\bfseries EndIf}        
     \State \quad\,\,\, $m_{s_t,a_t} \gets m_{s_t,a_t}+1$
   \State {\bfseries until} $m_{s,a} \ge  m$ $ \forall  (s,a) \in \cS \times \cA $
 \State {\bfseries Output:} $D$
\end{algorithmic}
\hrule
\end{minipage}
\vspace{-.3in}
\end{wrapfigure}

One issue with Theorem~\ref{thm:savic} is that knowledge of constants such as $t_{\mathrm{cov}}$, $t_{\mathrm{hit}}$, $d^{\pi_b}_{\mathrm{min}}$, and $\sp{Q^\star}$ are required to set parameters for obtaining $\varepsilon$-optimal policy. We can alleviate this issue by restricting the MDP to be a communicating MDP.

\begin{fact}\label{fact:com}
In a communicating MDP, $\cS = \cR$, where $\cR$ is a single recurrent class, and this decomposition is invariant for all policy $\pi \in \Pi_+$.
\end{fact}

In a communicating MDP, as Fact~\ref{fact:com} shows every state is recurrent for any $\pi \in \Pi_+$. Thus, the first and third stages from our previous framework can be omitted. Also, we replace Algorithm~\ref{alg:sam_1} with Algorithm~\ref{alg:sam_2}, which automatically stops when the adapted number of samples has been collected. Then, following the framework of \citep{lee2025near}, we apply a doubling trick and an explicit stopping rule based on the empirical Bellman residual error $\sp{Q^{n_i}-T^{n_i}}$ as described in Algorithm~\ref{alg:SAVIC+}.
\begin{wrapfigure}{r}{0.4\textwidth}
\vspace{-.13in}
\captionsetup{type=algorithm}
\hrule
\vspace{2pt}
\begin{minipage}{0.4\textwidth}
  \caption{\mbox{\sc \mbox{\rm SAVIC\textbf{+}}}$( \varepsilon,\delta)$}
  \label{alg:SAVIC+}
\vspace{-.1in}
\hrule
\begin{algorithmic}
  \State {\bfseries Input:} $\varepsilon>0, \delta\in(0,1)$
        \State {\bfseries Initialize:}    
   $n_{\mathrm{esc}},n_{\mathrm{rec}}\!=\!0$, $n_i\! =\!2^i$,  $\delta_i = \delta/(5(i\!+\!2)\ln^2(i\!+\!2)) \forall i \in \NN$  
  \State {\bfseries for} $i=0,1,\dots$ \,\,\textbf{do} 
    \State \quad $\mbox{\sc\mbox{SAVIC}}( n_{\mathrm{esc}},n_{\mathrm{rec}}, n_i, \emptyset, \varepsilon, \delta_i)$
    \State \quad $=(Q^{n_i},T^{n_i}, \pi^{n_i} )$
     \State {\bfseries until} $\sp{Q^{n_i}-T^{n_i}}\leq 14\,\varepsilon$
\State {\bfseries Output:} $Q^{n_i},T^{n_i}, \pi^{n_i}$
\end{algorithmic}
\hrule
\end{minipage}
\vspace{.2in}
\end{wrapfigure}

The following theorem shows that SAVIC+ does not require any prior knowledge to obtain an $\varepsilon$-optimal policy.
\begin{theorem}\label{thm:savic+}
Consider a communicating MDP. Let $(Q^{N},T^{N},\pi^{N})$ be the output of $\mbox{\sc \mbox{\rm SAVIC\textbf{+}}}(\varepsilon/16,\delta/2)$ with $\varepsilon \le 1 \le \sp{Q^\star}$. Then, with probability at least $1-\delta$, we have
\vspace{-1ex}
\begin{equation*}
\infn{J^\star-J^{\pi_{N}}}\! \leq \sp{Q^N-\T(Q^N)}\leq \,\varepsilon,
\end{equation*}
with sample complexity
\[
\widetilde{O}\left(t_{\mathrm{cov}}\sp{Q^\star}/\varepsilon
+ 1/d^{\pi_b}_{\mathrm{min}} \sp{Q^\star}^2/\varepsilon^{2}\right),
\]
where $\widetilde{O}(\cdot)$ hides logarithmic factors including $\log(1/\delta)$.
\end{theorem}

To the best of our knowledge, SAVIC+ is the first model-free method that guarantees finite sample complexity from a single trajectory in communicating MDPs.

\section{Stochastic $\alpha$-clipped Policy Mirror Ascent}~\label{sec:pol}
In this section, we introduce policy gradient methods that guarantee finite sample complexity from a single trajectory in weakly communicating MDPs.

\subsection{Framework overview}

\begin{wrapfigure}{r}{0.58\textwidth}
\vspace{-.16in}
\captionsetup{type=algorithm}
\hrule
\vspace{2pt}
\begin{minipage}{0.58\textwidth}
  \caption{SCPMA$(\alpha, n_{\mathrm{esc}},n_{\mathrm{rec}},K,\pi_0, \{n_{k}\}^K_{k=0}, \varepsilon, \delta)$\\c.s. = constant step size, a.s. = adaptive step size}
\label{alg:SCPMA}
\vspace{-.1in}
\hrule
\begin{algorithmic}
\State \textbf{Input:} $\alpha\!\in\!(0,1/|\mathcal{A}|), n_{\mathrm{esc}},n_{\mathrm{rec}},K\!\in\!\NN,\pi_0\!\in\!\Pi_\alpha,\{n_k\}^{K-1}_{k=0}\!\subset\!\NN^K$, $\varepsilon>0, \delta\in(0,1)$
   \State $R=\mbox{\sc recur}(n_{\mathrm{esc}}, n_{\mathrm{rec}})$
     \State $\eta\hspace{0.8ex}=\ln(8K|R||\cA|(n\!+\!1)/\delta)$
\For{$k=0,1,\ldots,K-1$}
 \State{$\varepsilon' =
\begin{cases}
\varepsilon/(4(2C_\alpha+K+2)) & \text{c.s.} \\
\varepsilon/(16C_\alpha)& \text{a.s.}
\end{cases}$}
\State $\hat{Q}^{\pi_k} =\mbox{SAVIPE}( \pi_k,  n_k,  \varepsilon'/2,\delta/(2K), \eta,R)$
\For{$s \in R$}
        \State $\pi_{k+1}(\cdot \,|\, s)=
 \underset{p\in \cM_\alpha (\cA)}{\argmax}$
 \State $\left\{
 b_k \sum_{a \in \cA}\hat{Q}^{\pi_k}(s,a)p(a)
 - D\big(p(\cdot),\pi_{k}(\cdot \,|\, s)\big)\right\}$
\EndFor
\EndFor
\State \textbf{Output:} $\pi_K$ 
\end{algorithmic}
\hrule
\end{minipage}
\vspace{-.3in}
\end{wrapfigure}
Our novel policy-based method, SCPMA consists of two stages. During the first stage, the agent escapes from transient states and collects the full list of recurrent states through Algorithm~\ref{alg:rec}. Then, during the second stage, in the for-loop of Algorithm~\ref{alg:SCPMA}, we alternately perform policy evaluation and policy improvement at every iteration. 

For policy evaluation step, we utilize our novel policy evaluation algorithm, SAVIPE which approximates state-action value function given policy $\pi$ with an anchoring mechanism and a variance-reduction technique. Compared with SAVIC, SAVIPE uses a different coefficient $\beta_k=\frac{k}{k+1}$ and the estimator $\frac{1}{n+1}\sum^{n-1}_{k=0}T^k$ to efficiently approximate the state-action value function $\hat{Q}^{\pi_k} \approx Q_\cR^{\pi_k}$ (see Appendix~\ref{omit:sec:4}) as described in Algorithm~\ref{alg:SAVIPE}. With approximated state-action value function $\hat{Q}^{\pi_k}$, we perform policy improvement following the policy mirror ascent framework \cite{shani2020adaptive,xiao2022convergence}. Policy mirror ascent utilizes a Bregman divergence and the state-action value function to update the policy effectively, as formulated in the preliminaries section. Specifically, we apply $\alpha$-clipping to policy mirror ascent motivated by \citep{lee2025multi}, and to this end, we define
\[
\Pi_{\alpha} = \{ \pi \,|\, \pi (a \,|\, s) \ge \alpha\text{ for all }s\in \cS,\,a\in \cA\}, \qquad \cM_\alpha(\cA) =\{p \,|\, p(a) \ge \alpha \text{ for all }\,a\in \cA\}
\]

\begin{wrapfigure}{r}{0.57\textwidth}
\vspace{-.1in}
\captionsetup{type=algorithm}
\hrule
\vspace{2pt}
\begin{minipage}{0.57\textwidth}
 \caption{\sc \mbox{SAVIPE}$( \pi, n,  \varepsilon,\delta, \eta, R)$}
   \label{alg:SAVIPE}
\vspace{-.1in}
\hrule
\begin{algorithmic}
     \State {\bfseries Input:} $\pi \in \Pi_+$, $n\in\NN$, $\varepsilon\!>\!0,\delta\!\in\!(0,1)$, $\eta \in \real, R\!\subset\!\cS$  
 \State {\bfseries Initialize:}  
   $Q^0\!=\!\mathbf{0}_{R\times \cA}, T^{-1}\!=\!r_{R\times \cA}, h^{-1}\!=\!\mathbf{0}_R,$    $\beta_k=k/(k\!+\!1),  c_k=5(k\!+\!2)\ln^2(k\!+\!2)  \forall \, k \in \NN $
    \For{ $k=0,\ldots,n-1$ }
     \State $h^k(s)= Q^{k}(s,\pi(\cdot \,|\, s))\quad(\forall s\in R)$
     \State $d^k=h^k-h^{k-1}$
    \State $m_k= \max\{\lceil \eta\,n^2c_k\sp{d^k}^2/\varepsilon^{2}\rceil,1\}$
    \State $t_k\!=\!\left\lceil e^2( \ln\!|R||\cA|\!+\!e)\!\left( t_{\mathrm{cov}}\!+\!\frac{m_k-1}{d^{\pi_b}_{\mathrm{min}}}\right)
\right\rceil \!\lceil \ln(\frac{4n}{\delta})\rceil$
     \State $D^k=\mbox{\sc SAMPLE}( t_k, m_k, d^k, R)$ 
     \State $T^k=T^{k-1}+D^k$ 
            \State $Q^{k+1} =(1\!-\!\beta_{k+1})\,Q^0+\beta_{k+1}\, T^{k}$
   \EndFor
   \State {\bfseries Output:} $Q^{n}$ 
\end{algorithmic}
\hrule
\end{minipage}
\vspace{-.1in}
\end{wrapfigure}
with $\alpha \in (0,1/|\cA|), p \in \cM(\cA)$, and  
\[\pi_\alpha =\argmax_{\pi \in \Pi_\alpha} J^\pi.\] By a slight abuse of notation, let $d^\pi$ denote the unique stationary distribution of the state-transition matrix $P^{\pi}_{\scriptscriptstyle \cS}$ (see Appendix~\ref{ap:sa}). On $\Pi_\alpha$, we can ensure the finiteness of the coefficients appearing in the sample-complexity results of Theorem~\ref{sam:SCPMA} as the following fact shows.
\begin{fact}[Lemma~3, \cite{lee2025multi}]\label{fact:bound}
\!For $\alpha\!\in\!(0,\frac{1}{|\cA|})$, 
    \[ \max_{\pi \in \Pi_\alpha}\infn{Q^{\pi}} :=Q_{\alpha} < \infty, \]
    \[\max_{
s \in \cR,
\pi,\pi' \in \Pi_\alpha
} \infn{\frac{d^{\pi}(s)}{d^{\pi'}(s)}}\!\!:=\!C_{\alpha}<\infty.\]  \end{fact}  


Despite that SCPMA optimizes the average reward over the clipped policy set $\Pi_\alpha$, the following fact guarantees that the gap between $J^\star$ and $J^{\pi_\alpha}$ can be made arbitrarily small by properly choosing $\alpha$.

\begin{fact}[Lemma 25, \citep{lee2025multi}]\label{fact:per}
Consider a weakly communicating MDP. For $\pi \in \Pi_+$,
\begin{align*}
J^{\star}-J^{\pi} 
&=    \sum_{s \in \cR}
\sum_{a \in \cA}
d^{\pi}(s) \big(\pi_\star(a \mid s)-\pi(a \mid s)\big)Q^{\pi_\star}(s,a). 
\end{align*}    
\end{fact}

 As $\alpha \rightarrow 0$, we can find $\pi \in \Pi_\alpha$ such that $\pi_\star(a \mid s)-\pi(a \mid s)$ becomes arbitrarily small, while $\infn{Q^{\pi_\star}}$ and $\|d^{\pi}\|_1 (=1)$ remain constant. 

We note that our SCPMA can be viewed as an extension of $\alpha$-clipped policy mirror ascent \citep{lee2025multi} to the single-trajectory setting, incorporating a variance-reduced policy evaluation algorithm. In contrast, \citep{lee2025multi} considers the generative model setting and employs a naive policy evaluation algorithm.

\subsection{Finite sample complexity for weakly communicating MDPs}

To establish the sample complexity of SCPMA, we first ensure that with high probability SAVIPE provides a good approximation of the state-action value function. To this end, we define coefficients related to the target time in weakly communicating MDP for $\pi \in \Pi_+$ and recurrent class $\mathcal{R}$. 
\begin{align*}
t_{s,a} := \inf\{t\ge 0 \,|\, s_t, a_t = s,a\}, \,\,\, t^{\pi}_{\mathrm{tar}}
:= \!\!\!\!\!\!\!\!\!\!\!\!\!\!\!\sum_{(s,a),(s',a')\in\mathcal{R}\times \cA}\!\!\!\!\!\!\!\!\!\!\!\! \!\!\!d^{\pi}(s',a')\,\mathbb{E}_{\pi}\!\left[t_{s',a'} \,|\, s_0, a_0=s,a \right],  \,\,\, t_\alpha=\max_{\pi \in \Pi_\alpha}t^{\pi}_{\mathrm{tar}}.
\label{eq:Ctar}
\end{align*}
We note that these coefficients are always finite when the state and action space are finite \citep{levin2017Markov, lee2025multi}. 

\begin{proposition}\label{prop:SAVIPE}
Consider a weakly communicating MDP.  Let $n_{\mathrm{esc}} \ge\lceil e t_{\mathrm{hit}}\rceil \lceil\ln (\frac{4}{\delta})\rceil$, $n_{\mathrm{rec}}  \ge  \lceil  e t_{\mathrm{cov}}\rceil \lceil \ln (\frac{4}{\delta})\rceil$, $n_k \ge \left\lceil\frac{8 t^{\pi_k}_{\mathrm{tar}}\sp{Q^{\pi_k}}}{\varepsilon'}\right\rceil$,
and  $\hat{Q}^{\pi_k}$ be the iterates generated by SCPMA$(\alpha, n_{\mathrm{esc}},n_{\mathrm{rec}},K,\pi_0, \{n_{k}\}^K_{k=0}, \varepsilon, \delta)$.
Then, with probability at least $1-\delta$, $
\sp{\hat{Q}^{\pi_k}-Q_{\cR}^{\pi_k}}\le  \varepsilon'$ for all $k=0,\ldots,K-1$.
\end{proposition}

As clarified in the preliminaries section, policy gradient in weakly communicating MDPs require the state-action value function only on recurrent states. Based on Proposition~\ref{prop:SAVIPE} and convergence analysis of inexact $\alpha$-clipped policy mirror ascent in \cite{lee2025multi}, we establish the complexity bound for \mbox{\rm SCPMA} with constant and adaptive step sizes.

\begin{theorem}\label{sam:SCPMA}
Consider a weakly communicating MDP. Let $\alpha = \frac{\varepsilon}{2|\cA| \infn{Q^{\pi^\star}}}$,  $n_{\mathrm{esc}} =\lceil e t_{\mathrm{hit}}\rceil \lceil\ln (\frac{4}{\delta})\rceil$, $n_{\mathrm{rec}}  =  \lceil  e t_{\mathrm{cov}}\rceil \lceil \ln (\frac{4}{\delta})\rceil$, $n_k = \lceil 8 t^{\pi_k}_{\mathrm{tar}}\sp{Q^{\pi_k}}/\varepsilon'\rceil$, $
K= \left\lceil4\left( \frac{D_{d^{\pi_\alpha}}({\pi_\alpha}, \pi_0)}{b}
+ C_\alpha\infn{J^{\pi_\alpha}-J^{\pi_0}} \right)/\varepsilon \right\rceil$ . Then, for $\varepsilon \le 1 \le \sp{Q^\star}$ and  any $\pi_0 \in \Pi_\alpha$, with probability at least $1-\delta$, the iterates generated by SCPMA$(\alpha, n_{\mathrm{esc}},n_{\mathrm{rec}},K,\pi_0, \{n_{k}\}^K_{k=0}, \varepsilon, \delta)$ with constant step size $b>0$
generate an $\varepsilon$-optimal policy with sample complexity 
\[
\widetilde{O}\left(t_{\mathrm{hit}}+gt_{\mathrm{cov}}t_{\alpha}C_\alpha^2Q_\alpha^3/\varepsilon^2+(t_{\mathrm{cov}}t_{\alpha}g^3Q_\alpha^2+1/d^{\pi_b}_{\mathrm{min}}g C^4_\alpha t_{\alpha}^2 Q_\alpha^4)/\varepsilon^5+ 1/d^{\pi_b}_{\mathrm{min}} t_{\alpha}^4Q_\alpha^4g^5/\varepsilon^{9}\right)
\]
where $g=C_\alpha\infn{J^{\pi_\alpha}-J^{\pi_0}}+\frac{D_{d^{\pi_\alpha}}({\pi_\alpha}, \pi_0)}{b}.$ With adaptive step sizes satisfying $b_{k+1}(C_\alpha  - 1) \ge b_k C_\alpha >0$, and $
K= \left\lceil
\log \left(4(\infn{J^{\pi_\alpha}-J^{\pi_0}}
+ \frac{D_{d^{\pi_\alpha}}(\pi_\alpha, \pi_{0})}{b_0(C_\alpha-1)})/\epsilon
\right)/ \log(C_\alpha /  (C_\alpha-1)) \right\rceil,$ with probability at least $1-\delta$,
the iterates generate an $\varepsilon$-optimal policy with sample complexity 
\[
\widetilde{O}\left(t_{\mathrm{hit}}+K(t_\alpha t_{\mathrm{cov}} Q_\alpha^2C_\alpha^2/\varepsilon^{2}+ 1/d^{\pi_b}_{\mathrm{min}} t_\alpha^2Q_\alpha^4C_\alpha^4/\varepsilon^{4})\right),
\]
where $\widetilde{O}(\cdot)$ hides logarithmic factors including $\log(1/\delta)$.
\end{theorem}

Although the adaptive step size yields a more efficient sample complexity, it requires knowledge of $C_{\alpha}$ to set the step sizes. We lastly clarify that to the best of our knowledge, Theorem~\ref{sam:SCPMA} is the first finite sample complexity guarantee for policy gradient methods from a single trajectory in weakly communicating MDPs.

\section{Conclusion}
In this work, we propose novel model-free methods that achieve the first finite sample complexity from a single trajectory for weakly communicating average-reward MDPs. For communicating MDPs, we further introduce the first model-free method that requires no prior knowledge of problem-dependent quantities. One future research direction is to close the sample-complexity gap between the value-based and policy-based methods by using a refined variance-reduction technique. Another  future direction is to extend our results to parameterized function-approximation settings, where handling representation error and stability under single trajectory sampling poses additional challenges.


\newpage
\bibliographystyle{abbrv}
\bibliography{neurips_2026}

\newpage
\appendix


\section{Omitted proofs in Section~\ref{sec:sin}}

\subsection{Proof of Lemma~\ref{lem:hit}}
\begin{proof}
    First, by Markov inequality, $P(\tau_{\mathrm{hit}} \ge \lceil e t_{\mathrm{hit}}\rceil) \le e^{-1}$. By induction, we will show $P(\tau_{\mathrm{hit}} \ge m\lceil e t_{\mathrm{hit}}\rceil) \le e^{-m}$ for $m \in \mathbb{N}$. 
    \begin{align*}       \expec[\mathbf{1}_{\tau_{\mathrm{hit}} \ge m\lceil e t_{\mathrm{hit}}\rceil}]&=\expec[\mathbf{1}_{\tau_{\mathrm{hit}} \ge (m-1)\lceil e t_{\mathrm{hit}}\rceil} \expec  [\mathbf{1}_{\tau_{\mathrm{hit}} \ge m\lceil e t_{\mathrm{hit}}\rceil} \,|\, \cF_{(m-1)\lceil e t_{\mathrm{hit}}\rceil}] ]
    \\&=\expec[\mathbf{1}_{\tau_{\mathrm{hit}} \ge (m-1)\lceil e t_{\mathrm{hit}}\rceil} P(\tau_{\mathrm{hit}} \ge \lceil e t_{\mathrm{hit}}\rceil \,|\, s_0) ]
    \\& \le \expec[\mathbf{1}_{\tau_{\mathrm{hit}} \ge (m-1)\lceil e t_{\mathrm{hit}}\rceil}]e^{-1}
    \\& \le e^{-m}
    \end{align*}
    where second equality is from Markov property and defintion of hitting time and second inequality is from induction.
\end{proof}

\subsection{Proof of Lemma~\ref{lem:sam_bd}}
For Markov chain define $N_s (t)$ as counting measure of state $s$ appearing trajectory until $t$. Also, define $k$ cover time and and expected $k$ cover time as  
\[\tau^k_{\mathrm{cov}} = \inf \{ t \,|\,  N_s(t) \ge k \,\forall s\in \cS\}, \qquad t^{k}_{\mathrm{cov}}=\max_{s \in \cS} \expec [\tau^k_{\mathrm{cov}} \,|\, s_0=s].\]
Then, following fact holds.
\begin{fact}[\cite{chan2021learning}]~\label{fact:kcov}
 Consider a single trajectory generated by a Markov process with a single recurrent class. For $k\in \mathbb{N}$ and  any initial state, $\prob (\tau^k_{\mathrm{cov}} \ge m \lceil et^k_{\mathrm{cov}} \rceil)\le e^{-m}$ and $t^k_{\mathrm{cov}} \le  (e \ln|\cS|+e^2)( t_{\mathrm{cov}}+(k-1)/d_{\mathrm{min}})$ . 
\end{fact}
Using Fact~\ref{fact:kcov}, we first prove our key lemma.

\begin{proof}
Define $\cF_{i} = \sigma(s_0, s_1, \dots, s_{\tau_{i+1}})$, and let $f_i=  f( s_{\tau_{i}+1})- \expec \left[ f(s_{\tau_i+1})\right]$. Then $f_i$ is $\cF_{i}$-measurable. Define $\cF_{i-1}$ measurable function $H_i = \mathbf{1}_{\{\tau_i \le N\}}$. Then, $\sum^I_{i=1}f_i=\sum^M_{i=1}f_i H_i$. By Hoeffding's lemma and strong Markov property, we have
\[\expec[\mathrm{exp}(\lambda f_i H_i) \,|\, \cF_{i-1}] \le \expec[\mathrm{exp}( H^2_i \lambda^2 (a-b)^2/8)\,|\, \cF_{i-1}] \le \mathrm{exp}( \lambda^2 (a-b)^2/8).\]
Thus, by induction,

\begin{align*}
    \expec\left[\exp \left(\frac{\lambda}{M}\sum^I_{i=1}(f(s_{\tau_{i}+1}) -\expec [f(s_{\tau_{i}+1})]\right)\right]  =\expec\left[\exp \left(\frac{\lambda}{M}\sum^M_{i=1}f_iH_i\right)\right]
\le \mathrm{exp}\left( \frac{ \lambda^2(a-b)^2}{8M}\right).
\end{align*}

By Fact~\ref{fact:kcov} and direct calculation, we obtain second statement.
\end{proof}

\section{Omitted proofs in Section~\ref{sec:savi}}\label{omit:sec3}

\subsection{Proof of Lemma~\ref{lem:pur}}
\begin{proof}
Since $\pi_\epsilon \in \Pi_+$, we have
\begin{align*}
    J^{\pi_\epsilon} &= P^{\pi_\epsilon}_\star r
    \\&= P^{\pi_\epsilon}_\star (r+P^{\pi_\epsilon} Q-Q)
    \\&= P^{\pi_\epsilon}_\star (\T^{\pi_\epsilon} Q-Q)
    \\&\ge \min_{(s,a) \in \cR \times \cA} {(\T^{\pi_\epsilon} Q-Q)(s,a)} 
    \\&\ge \min_{(s,a) \in \cR \times \cA}{(\T^{\pi} Q-Q)(s,a)} -\epsilon \max_{(s,a) \in \cR \times \cA}|Q(s,a)|,
\end{align*}
where second equality is from  $P^{\pi_\epsilon}_\star P^{\pi_\epsilon} = P^{\pi_\epsilon}_\star$, first inequality follows from  $P^{\pi_\epsilon}_\star((s',a'), (s,a))=0$ for $(s,a) \in \cT \times \cA$, and second inequality is from $\max_{(s,a) \in \cR \times \cA}|\T^\pi Q(s,a)-\T^{\pi_\epsilon} Q(s,a)|\le \epsilon \max_{(s,a) \in \cR \times \cA}|Q(s,a)|$.

By Fact~\ref{fact:per}, for any $\epsilon'>0$, there exist $\pi' \in \Pi_+$ such that $J^\star-J^{\pi'} \le \epsilon'$. Then, with similar argument,
\begin{align*}
    J^{\pi'} = P^{\pi'}_\star r = P^{\pi'}_\star (r+P^{\pi'} Q-Q) \le \max_{(s,a) \in \cR \times \cA} {(\T^{\pi'} Q-Q)(s,a)} \le \max_{(s,a) \in \cR \times \cA} {(\T Q-Q)(s,a)} 
\end{align*}
where the last inequality is from the definition of $\T$. Therefore,  we have $    J^{\star} \le   \max_{(s,a) \in \cR \times \cA} {(\T Q-Q)(s,a)}$ and this implies
\begin{align*}
    \infn{J^\star-J^{\pi_\epsilon}} &\le \max_{(s,a) \in \cR \times \cA} {(\T^{\pi} Q-Q)(s,a)}-\min_{(s,a) \in \cR \times \cA} {(\T^{\pi} Q-Q)(s,a)} +  \epsilon \max_{(s,a) \in \cR \times \cA}|Q(s,a)|.
\end{align*}
\end{proof}

\subsection{State action transition probability in weakly communicating MDP }\label{ap:sa}
Given MDP, define $P_{\scriptscriptstyle \cS \cA}^\pi(s,a)$ and  $P^\pi_{\scriptscriptstyle\cS}(s)$ as
\begin{align*}
    P^{\pi}_{\scriptscriptstyle \cS }(s\rightarrow s')&=
\mathrm{Prob}(s\rightarrow s'\,|\,
a \sim \pi(\cdot\,|\,s), s'\sim P(\cdot\,|\,s,a))
\\
P^{\pi}_{\scriptscriptstyle \cS \cA}((s,a)\rightarrow (s',a'))&=
\mathrm{Prob}((s,a)\rightarrow (s',a')\,|\,
s'\sim P(\cdot\,|\,s,a),a' \sim \pi(\cdot\,|\,s')).
\end{align*}
In the maintext, we focus on $P_{\scriptscriptstyle \cS \cA}^\pi$. Following lemma clarifies correspondence between the recurrent-transient classifciation of $P^{\pi}_{\scriptscriptstyle \cS }$ and $P_{\scriptscriptstyle \cS \cA}^\pi$.
\begin{lemma}~\label{lem:inv}
    For $\pi \in \Pi_+$, if $ \{\cR_i\}^k_{i=1}, \cT $ are irreducible recurrent classes and transient set of     $P^{\pi}_{\scriptscriptstyle \cS}$ respectively, then,  $ \{\cR_i \times \cA\}^k_{i=1}, \cT\times \cA $ are irreducible recurrent classes and transient set of     $P^{\pi}_{\scriptscriptstyle \cS\cA}$  respectively.
\end{lemma}
\begin{proof}
    Suppose $s \in \cR_i$. Then, for any $a \in \cA$, if $s' \not\in \cR_i$, $P(s' \,|\,s,a ) = 0$ since $\pi \in \Pi_+$ and finite irreducible recurrent class is closed. Hence, there exist $s'$ such that $s'\in \cR_i$ and $P(s' \,|\,s,a ) \neq 0$. By definition of communication class, there exist $n$ satisfying $(P^{\pi}_{\scriptscriptstyle \cS})^{n}(s',s'') \neq 0 $ for any $s''\in \cR_i$. Then, $(P^{\pi}_{\scriptscriptstyle \cS \cA})^{n+1} = P(P^{\pi}_{\scriptscriptstyle \cS})^{n}\Theta_\pi $ implies that $(P^{\pi}_{\scriptscriptstyle \cS \cA}((s,a)\rightarrow (s'',a')))^{n+1} \neq 0$ for all $a' \in \cA$. Those show that $\cR_i \times \cA$ is closed communicating class of     $P^{\pi}_{\scriptscriptstyle \cS\cA}$. Lastly, suppose $s \in \cT$. Then, for any $ s\in \cT, a \in \cA$, $\sum^\infty_{n=0}(P^{\pi}_{\scriptscriptstyle \cS \cA})^{n}(s,a) = I+P(\sum^\infty_{n=0} (P^{\pi}_{\scriptscriptstyle \cS})^{n})\Theta_\pi(s,a) <\infty $  since $(\sum^\infty_{n=0} (P^{\pi}_{\scriptscriptstyle \cS})^{n})(s',s'') <\infty$ for all $s',s'' \in \cT$ (\citep[Lemma 8.3.20]{berman1994nonnegative}). This shows that $\cT\times \cA$ is transient set.
\end{proof}
In particular, Lemma~\ref{lem:inv} implies that  if $P^\pi_{\cS}$ has a single recurrent class, then  $P^\pi_{\cS \cA}$ has a single recurrent class as well. For a single trajectory, $(s_{i+1},a_{i+1})$ can be considered as a state of transition probability $P^\pi_{\cS \cA}$, i.e., $(s_{i+1},a_{i+1}) \sim P^\pi_{\cS \cA} ( \cdot, \cdot \,|\, s_{i},a_{i})$. Hence, facts and lemmas in section~\ref{sec:sin} continue to hold with respect to $P^{\pi_b}_{\cS \cA}$ for $\pi_b \in \Pi_+$.

\subsection{Proof of Proposition~\ref{prop:savic}}

Let $L_k=\sum^{k-1}_{i=0}(t_i+1)$ and $\mathcal{F}_k=\sigma((s_i,a_i)\,|\, 0\le i\le L_{k+1})$ denote the natural filtration generated by the sampling process in {\rm SAVIC}, and $\prob(\cdot)$ the probability distribution over the trajectories.
Then, $T^k$ is $\mathcal{F}_k$-measurable  whereas $Q^k$, $h^k$, $d^k$ and $m_k$, being functions of $T^{k-1}$, are $\mathcal{F}_{k-1}$-measurable. 

\begin{proof}
First, by Lemma~\ref{lem:hit}, Fact~\ref{fact:cov}, and union bounds, with $1-\delta/2$ probability, $s_{n_{\mathrm{esc}}+n_{\mathrm{rec}}}\in \cR$ and $R=\cR$. 

Suppose that $s_{n_{\mathrm{esc}}+n_{\mathrm{rec}}} \in \cR$ and $R=\cR$ so that we focus on $\cR \times \cA$, and for simplicity, consider $s_{n_{\mathrm{esc}}+n_{\mathrm{rec}}}$ as $s_0$. Let $s_{\tau^k_{j}(s,a)+1}\sim \cP(\cdot \,|\, s,a)$ for $j=1,\ldots,I^k_{s,a}$ be the samples at the $k$-th iteration where $ \tau^k_i(s,a) = \inf \{l > \tau^k_{i-1} : (s^l, a^l)=(s,a)\} $  for $(s,a)\in\cR\times \cA$ and $\tau^k_0 = L_k-1$.  Define  $I^k_{s,a}=\min \{m_k, \max \{j \,|\, \tau^k_j(s,a) \le L_k+t_k\}\}$,
     \begin{align*}
         Y^{k}(s,a) =\mbox{$\frac{1}{m_k}\sum_{j=1}^{I^k_{s,a}}\big(d^k(s_{\tau^k_{j}(s,a)+1})$}-\cP d^k(s,a)\big)\qquad\forall (s,a)\in\cR\times\cA, 
     \end{align*} 
     and $X^k=\sum_{i=0}^kY^i$. We proceed to estimate $\prob(\|X^k\|_\infty\geq\varepsilon)$
by adapting the arguments of Azuma-Hoeffding's inequality with stopping time argument.
     Note that $d^k$ and $m_k$ are $\mathcal{F}_{k-1}$-measurable. 
From Markov's inequality and the tower property of conditional expectations we get that for each $(s,a)\in\cR\times \cA$ and $\lambda>0$
\begin{align}\nonumber
   \mbox{$ \prob( X^k(s,a) \ge \varepsilon)$} 
   & \le e^{-\lambda\varepsilon}\,\expec\big[\exp(\lambda\, \mbox{$X^k(s,a)$})\big]\\
   &= e^{-\lambda\varepsilon}\,\expec\big[\exp(\lambda\,\mbox{$X^{k-1}(s,a)$}) \;\expec[\exp(\lambda \, Y^{k}(s,a)) \,|\, \mathcal{F}_{k-1}]\big].\label{eq:a1}
\end{align}
Conditionally on $\cF_{k-1}$, for fixed $j$, $d^k(s_{\tau^k_{j}(s,a)+1})-\cP d^k(s,a)$ is random variables with zero mean and its absolute value is upper bounded by $ \sp{d^k}$. Then, by Lemma~\ref{lem:sam_bd}, we have 
\begin{align}
    \expec\big[\exp(\lambda\,Y^{k}(s,a)) \,|\, \mathcal{F}_{k-1}\big] 
    \le \exp\big(\mbox{$\frac{1}{2}$}\lambda^2(\sp{d^k})^2/m_k\big). \label{eq:a2}
\end{align}

Using \eqref{eq:a1} and \eqref{eq:a2}, together with definition of $m_k$, 
a simple induction yields $$\expec[\exp(\lambda\,X^k(s,a))]\leq\exp\big(\mbox{$\frac{1}{2}\lambda^2\varepsilon^2\sum^k_{i=0} c^{-1}_i/\eta$}\big).$$ Then, since $\sum^{\infty}_{i=0}c^{-1}_i\leq\frac{1}{2}$ we get
$\prob(X^k(s,a)\geq\varepsilon)\leq \exp\big(\mbox{$-\lambda\,\varepsilon+\frac{1}{4}\lambda^2\varepsilon^2/\eta$}\big)$ and 
taking $\lambda=2\eta/\varepsilon$ we deduce
$$\prob(X^k(s,a)\geq\varepsilon)\leq \exp(-\eta)=\delta/(8|\cR||\cA|(n\!+\!1)).$$
A symmetric argument yields the same bound for $\prob(X^k(s,a)\leq-\varepsilon)$ so that 
$\prob(|X^k(s,a)|\geq\varepsilon)\leq \delta/(4|\cR||\cA|(n\!+\!1))$.
Applying a union bound over all $(s,a)\in\cR\times\cA$ we get $\prob(\|X^k\|_\infty\geq\varepsilon)\leq \delta/(4(n\!+\!1))$, and then a second union bound over $k$ and taking the complementary event gives $\prob(\bigcap_{k=0}^n\{\|X^k\|_\infty\leq\varepsilon\})\geq 1- \delta/4$.

Next, by the second statement of Lemma \ref{lem:sam_bd} and union bound over $k$,\[\prob\left(\bigcap_{0 \le k \le n} \{I^k_{s,a}=m_k \, \forall (s,a) \in \cR \times \cA\} \right)\geq 1- \delta/4.\]
 In this event, since $h^{-1}=0$, by telescoping $\T Q^k=r+\cP h^k=r+\sum_{i=0}^k\cP d^i$ and we get 
\[(T^k-\T Q^k)(s,a)=\sum_{i=0}^k(D^i-\cP d^i)(s,a)=\sum_{i=0}^kY^i(s,a)=X^k(s,a) \quad \forall (s,a)\in \cR\times \cA.\] Then, intersection of $\bigcap_{0 \le k \le n}\{\|X^k\|_\infty\leq\varepsilon\}$ and $\bigcap_{0 \le k \le n} \{I^k_{s,a}=m_k \, \forall(s,a) \in \cR \times \cA\}$ gives
 \[\prob\left(\max_{(s,a)\in \cR \times \cA} |T^k-\T(Q^k)|(s,a) \le  \epsilon \,\, \forall 0 \le k \le n \right) \ge 1-\frac{\delta}{2}\]
The conclusion follows by union bound and conditioning on first sampling process.
\end{proof}

\subsection{Proof of Theorem~\ref{thm:savic}}

We first introduce following lemma. 
\begin{lemma}\label{lem:rec}
    Consider a weakly communicating $MDP(\cS,\cA, P, r)$, and restricted recurrent $MDP(\cR,\cA, P_\cR, r_{\cR \times \cA})$ where transition probability $P_\cR: \cR\times \cA \rightarrow \cM(\cR)$ is defined by $P_\cR(s' \,|\, s,a) = P(s' \,|\, s,a)$ for all $s,s' \in \cR, a \in \cA$ and $ r_{\cR \times \cA}$ is defined by $r_\cR(s,a) = r(s,a)$ for all $s \in \cR, a \in \cA$. Let $\pi^\star_\cR$ and $Q_{\cR}^\star$ be optimal policy and corresponding value function in restricted MDP. Then, there exist $\pi^\star$ and corresponding value function $Q^\star(s,a)$ for $MDP(\cS,\cA, P, r)$ such that $\pi^\star_\cR(a \,|\,s)=\pi^\star(a \,|\,s)$ and $Q^\star_{\cR}(s,a)=Q^\star(s,a)$ for all $s, a \in\cR \times \cA$.
\end{lemma}
\begin{proof}
    First, we note that $P_\cR$ is indeed transition probability since  $P(s' \,|\, s,a) =0$ for $s' \in \cR^\complement, s\in \cR $ as showed in the proof of Lemma~\ref{lem:inv}. This implies that $P^\pi((s',a'),(s,a))=P_{\cR}^{\pi_\cR}((s',a'),(s,a))$ and  $P^\pi_\star((s',a'),(s,a))=P_{\cR,\star}^{\pi_\cR}((s',a'),(s,a))$ for all $s,s'\in \cR, a,a'\in \cA$  by property of product and sum of block matrix $\begin{bmatrix}
       A & \mathbf{0}
        \\B& C
    \end{bmatrix}$ where $A,C$ are square matrices. So, $J_\cR^{\pi_\cR}(s,a)=J^{\pi}(s,a)$ for any $\pi$ satisfying  $\pi_\cR(a \,|\,s)=\pi(a \,|\,s)$  for all $s, a \in\cR \times \cA$ and this implies $J^\star(s,a)=J^{\pi^\star_\cR}(s,a)=J^{\pi^\star}(s,a)$ for policy $\pi^\star$ satisfying $\pi^\star_\cR(a \,|\,s)=\pi^\star(a \,|\,s)$  for all $s, a \in\cR \times \cA$. Since  for any policy $\pi\in \Pi$ (not $\Pi_+$), $\cT\times \cS$ for $P^\pi$ continues to be transient set through the same argument in Lemma~\ref{lem:rec} with Fact~\ref{fact:rec} and  using the fact that optimal average reward is constant in weakly communicating MDP and, we can show that $J^{\pi^\star}(s,a)$ is constant vector and $\pi^\star$ is optimal policy in oringal MDP. Lastly, by property of inverse and product of block matrix $\begin{bmatrix}
       A & \mathbf{0}
        \\B& C
    \end{bmatrix}$ where $A,C$ are square matrices, we also have $Q^\star_{\cR}(s,a)=Q^\star(s,a)$ for all $s, a \in\cR \times \cA$.  
\end{proof}
We now introduce following key lemma. 
\begin{lemma}\label{lem:sam}
With the coefficients of Proposition~\ref{prop:savic}, let $(Q^n,T^n,\pi_{n})$ be the iterates generated  by  \mbox{SAVIC}$( n_{\mathrm{esc}},n_{\mathrm{rec}}, n, \{t_k\}^n_{k=0}, \varepsilon, \delta)$. Then, with probability at least $1-\delta$, we have
\vspace{-1ex}
\[  \sp{Q^n-\T Q^n}\leq \frac{8\sp{Q^0-Q_\cR^\star}}{n+2}+4\varepsilon.\]
\end{lemma}
\begin{proof}
As proof in Proposition~\ref{prop:savic} shows, with probability at least $1-\delta$, $R=\cR$ and $I^k_{s,a}=m_k$ for all $(s,a) \in \cR \times \cA $. On the event of Proposition~\ref{prop:savic}, combining these with the proof of the Theorem 3.2 of \citep{lee2025near}, we have upper bound of Bellman error in restricted recurrent $MDP(\cR,\cA, P_\cR, r_{\cR \times \cA})$ where transition probability $P_\cR: \cR\times \cA \rightarrow \cM(\cR)$ is defined by $P_\cR(s' \,|\, s,a) = P(s' \,|\, s,a)$ for all $s,s' \in \cR, a \in \cA$ and $ r_{\cR \times \cA}$ is defined by $r_\cR(s,a) = r(s,a)$ for all $s \in \cR, a \in \cA$.
\end{proof}
Now we are ready to prove Theorem~\ref{thm:savic}
\begin{proof}
With coefficients in Theorem~\ref{thm:savic}, we have $\sp{Q^n-\T(Q^n)}\le \frac{\varepsilon}{2}$ by Lemma~\ref{lem:rec} and ~\ref{lem:sam}, and we get  $\infn{J^\star-J^{\pi^n_{ \varepsilon}}} \le \varepsilon $ by Lemma~\ref{lem:pur}.

Now, to estimate the total number of samples $\sum_{k=0}^n t_k$ 
we recall that $m_k=\max\{\lceil\eta c_k\sp{d^{k}}^2/\varepsilon^{2}\rceil,1\}$ which can be bounded as $t_k\leq 1+e^2( \ln|\cR||\cA|+e)( t_{\mathrm{cov}}+(m_k-1)/d^{\pi_b}_{\mathrm{min}})\log \left(\frac{4(n+1)}{\delta }\right)$. Then 
\begin{align*}
    \sum^{n}_{k=0} t_k &\le  n+1+\log \left(\frac{4(n+1)}{\delta }\right) \left(e^2( \ln|\cR||\cA|+e)(n+1)t_{\mathrm{cov}}  +1/d^{\pi_b}_{\mathrm{min}}\sum^n_{k=0}m_k \right)\\
    &= \widetilde{O}\left(t_{\mathrm{cov}}\sp{Q^\star}/\varepsilon+ 1/d^{\pi_b}_{\mathrm{min}} \sp{Q^\star}^2/\varepsilon^{2}\right)
\end{align*}
since $ \sum^n_{k=0}m_k=O(\mbox{$\eta\ln^3(n+2)\sp{Q_\cR^\star}^2/\varepsilon^{2}+\eta \,n^2\ln^2(n+2)$})$ from the proof of the Theorem 3.2 of \citep{lee2025near}, and the fact that $1/d^{\pi_b}_{\mathrm{min}}, t_{\mathrm{cov}} \ge |\cR||\cA|$.  Lastly, by adding $n_{\mathrm{hit}}+n_{\mathrm{rec}}$, we obtain desired result.

\end{proof}
\subsection{Proof of Theorem~\ref{thm:savic+}}
\begin{proof}
For the proof, we again follow the proof framework of \cite{lee2025near}.

The stopping time of \mbox{\rm SAVIC\textbf{+}} is the random variable 
\begin{equation*}
 I =\inf\{ i \in \mathbb N \,:\, \sp{Q^{n_i}-T^{n_i}} \le 14\,\varepsilon\}
\end{equation*}
with $T^{n_i}$ and $Q^{n_i}$'s the iterates generated in each loop of \mbox{\rm SAVIC\textbf{+}}$(Q^0,\varepsilon,\delta)$..

    We let $i_0 \in \NN$ be the smallest integer satisfying $n_{i_0}\ge\sp{Q^{*}}/\varepsilon$, so that 
either $i_0=0$ and $n_{i_0}=1$ or $n_{i_0-1}=n_{i_0}/2<\sp{Q^{*}}/\varepsilon$, which combined imply $n_{i_0}\le 2(1+\sp{Q^{*}}/\varepsilon)$.
\begin{lemma}\label{lem:savic+}
Consider a communicating MDP. Let $(Q^{N},T^{N},\pi^{N})$ be the output of $\mbox{\sc \mbox{\rm SAVIC\textbf{+}}}(\varepsilon/16,\delta/2)$ with $\varepsilon \le 1 \le \sp{Q^\star}$. Then, with probability at least $1-\delta/2$ we have $I \le i_0$ and
\vspace{-1ex}
\begin{equation*}
\infn{J^\star-J^{\pi_{N}}}\! \leq \sp{Q^N-\T(Q^N)}\leq \varepsilon.
\end{equation*}
\end{lemma}
\begin{proof}
    First, by Fact~\ref{fact:com}, every state is recurrent in communicating MDPs, and thus $\mbox{\sc SAMPLE+}( m_k, d^k)$ almost surely provide i.i.d. $m_k$ samples for each state and action. Then, the proof of Corollary 3.5 of ~\cite{lee2025near}  can be directly applied for  $\mbox{\sc \mbox{\rm SAVIC\textbf{+}}}(\varepsilon/16,\delta/2)$ and desired result follows. 
\end{proof}
    Due to Lemma~\ref{lem:savic+}, we directly obtain convergence result in Theorem~\ref{thm:savic+}. 
    
 We denote the number of  samples used during  the execution of \mbox{\sc\mbox{SAVIC}}$( n_{\mathrm{esc}},n_{\mathrm{rec}}, n_i, \emptyset, \varepsilon, \delta_i)$ in the $i$-th loop of  \mbox{\rm SAVIC\textbf{+}}  by $M_i$, so that the total sample complexity is $M=\sum^{I}_{i=0}M_i$.

Note that by Lemma~\ref{lem:sam_bd}, with at least $1-\delta_i$ probability,
length of trajectory sampled from $\mbox{\sc SAMPLE+}( m_k, d^k)$ in the $i$-th loop of  \mbox{\rm SAVIC\textbf{+}} is smaller than $ t_k^i=1+e^2( \ln|\cS||\cA|+e)( t_{\mathrm{cov}}+(m_k-1)/d^{\pi_b}_{\mathrm{min}})\log \left(\frac{4(n_i+1)}{\delta_i }\right)$ for $0\le k\le n_i$ and $s,a \in \cS \times \cA$.  Then, based on sample complexity bound analysis in the proof of the Theorem~\ref{thm:savic} and Corollary 3.5 of ~\cite{lee2025near} and union bound,  with at least $1-\delta/2$ probability, 
\begin{align*}
M&\leq
\sum^{i_0}_{i=0}\sum^{n_i}_{k=0} t_k^i.
\\&=\sum^{i_0}_{i=0}(i+1)+\log \left(\frac{4(n_i+1)}{\delta_i }\right) \bigg(\sum^{i_0}_{i=0}e^2( \ln|\cS||\cA|+e)(n+1)t_{\mathrm{cov}}\\&+ 1/d^{\pi_b}_{\mathrm{min}}\sum^{i_0}_{i=0}\left( \mbox{$\eta_i\ln^3(n_i+2)\sp{Q^\star}^2/\varepsilon^{2}+\eta_i\,n_i^2\ln^2(n_i+2)$}\right)\bigg)
\\&=\widetilde{O}\left(t_{\mathrm{cov}}\sp{Q^{*}}/\varepsilon + 1/d^{\pi_b}_{\mathrm{min}}\sp{Q^\star}^2/\varepsilon^2\right)
\end{align*}
where $\eta_i= \ln(8|\cS||\cA|(n_i\!+\!1)c_i/\delta),$ 
is the parameter used in the $i$-th internal cycle of \mbox{\rm SAVIC\textbf{+}}. Lastly, by union bound of sample complexity bound and convergence results, we have Theorem~\ref{thm:savic+}.

\end{proof}


\section{Omitted proofs in Section~\ref{sec:pol}}\label{omit:sec:4}
\subsection{Proof of Proposition~\ref{prop:SAVIPE}}
In this section, we suppose that $\pi \in \Pi_+$ and MDP is weakly communicating. We consider restricted recurrent $MDP(\cR,\cA, P_\cR, r_{\cR \times \cA})$ where transition probability $P_\cR: \cR\times \cA \rightarrow \cM(\cR)$ is defined by $P_\cR(s' \,|\, s,a) = P(s' \,|\, s,a)$ for all $s,s' \in \cR, a \in \cA$ and $ r_{\cR \times \cA}$ is defined by $r_\cR(s,a) = r(s,a)$ for all $s \in \cR, a \in \cA$. Let $\pi_{\cR}$ and $Q^\pi_{\cR}$ be policy and corresponding state-action value function, respectively. Let $d^\pi_\cR(=P^\pi_{\cR \star})$ be the unique stationary distribution of $P^\pi_\cR$.

We first introduce lemma related to target time. 
\begin{lemma}\label{lem:tar} Consider a weakly communicating. Then, we have 
\[
\left\|
\frac{1}{k+1}\sum_{j=0}^{k}\bigl(P_{\cR}^{\pi}\bigr)^{j}(\cdot, (s,a))\;-\;d^{\pi}(\cdot)
\right\|_{1}
\;\le\;
\frac{2 t^{\pi}_{\mathrm{tar}}}{k+1} \qquad \forall s,a \in \cR\times\cA.
\]
\end{lemma}
\begin{proof}
Combining the proof of Lemma~\ref{lem:rec}   and Corollary~3 of \cite{roberts1997shift}, we have desired result.
\end{proof}

Consider Anc-QI with Bellman consistency operator :
\[Q^k = \beta_k\T^\pi Q^{k-1}+(1-\beta_k) Q^0.\] 

We now present convergence result of Anc-QI.
\begin{lemma}~\label{lem:sub}
    Anc-QI with Bellman consistency operator, $\beta_k=\frac{k}{k+1}$, and $Q^0=0$ exhibits the rate
    \[\sp{Q^k-Q_{\cR}^\pi} \le \frac{4 t^\pi_{\mathrm{tar}}\sp{Q_{\cR}^\pi}}{k+1}.\]
\end{lemma}
\begin{proof}
    By definition of Anc-QI, $Q^k=\frac{1}{k+1} \sum^k_{i=0}(\T^\pi)^i Q^0$, and
    \begin{align*}
        (\T^\pi)^i Q^0 &= \sum^{i-1}_{j=0}(P_{\cR}^\pi)^j r_{\cR} 
        \\&= \sum^{i-1}_{j=0}(P_{\cR}^\pi)^j(J^\pi+(I-P_{\cR}^\pi)Q_{\cR}^\pi)
        \\&=iJ^\pi +Q_{\cR}^\pi-(P_{\cR}^\pi)^iQ_{\cR}^\pi,
    \end{align*}
where second and third equalities come from Bellman equations. Utilizing the fact that $P^\pi_{{\cR},\star} Q_{\cR}^\pi=0$ from Section A.5 of \citep{Puterman2014} and the proof of Lemma~\ref{lem:rec}, we have
        \[Q^k-Q^\pi_{\cR}-\frac{k}{2} J^\pi= \left( P^\pi_{{\cR}\star}- \frac{\sum^k_{i=0}(P_{\cR}^\pi)^i}{k+1}\right)Q_{\cR}^\pi.\]
         By taking $\sp{\cdot}$ and using properties of span seminorm $\sp{v+c\mathbf{1}}=\sp{v}$ for $c \in \real$ and $\sp{v} \le 2\infn{v}$m Lemma~\ref
        {lem:tar}, the fact that $J^\pi$ is constant vector, we have desired result.
\end{proof}

We introduce the following lemma, which is analogous to Proposition~\ref{prop:savic}. Its proof follows directly from that of Proposition~\ref{prop:savic} after appropriately adapting the coefficients (note that unlike SAVIC, $m_k= \max\{\lceil \eta\,c_k\sp{d^k}^2(n/\varepsilon)^{2}\rceil,1\}$ in SAVIPE).

\begin{lemma}\label{lem:err}
Consider a weakly communicating MDP. Let $n_{\mathrm{esc}} \ge\lceil e t_{\mathrm{hit}}\rceil \lceil\ln (\frac{4}{\delta})\rceil, n_{\mathrm{rec}}\ge$
$ \lceil  e t_{\mathrm{cov}}\rceil \lceil \ln (\frac{4}{\delta})\rceil$, $n=\left\lceil\!\frac{8 t^{\pi_k}_{\mathrm{tar}}\sp{Q^\pi}}{\varepsilon}\right\rceil$,
and let $T^k, Q^k$ be the iterates generated by 
{\sc \mbox{SAVIPE}}$( \pi, n,  \varepsilon/2,\delta, \eta, R).$
Then, with probability at least $1-\delta$, $R=\cR$, $I^k_{s,a}=m_k$, and $
\infn{T^k-\T^\pi(Q^k) } \le  \varepsilon/(2n) $
for all $k=0,\ldots,n-1$ and $(s,a) \in \cR \times \cA$. 
\end{lemma}

Now, consider inexact Anc-QI with Bellman consistency operator:  
\[Q^k = \frac{k}{k+1}(\T^\pi Q^{k-1}+\epsilon_k) + \frac{1}{k+1}Q^0\]
where $\epsilon_k \in \real^{\cR\times \cA}$ represents approximation error of Bellman operator. Then, by direct calculation, 
\[Q^k=\frac{1}{k+1}\sum^k_{i=0}(\T^\pi)^i Q^0+\frac{k+1-i}{k+1}\sum^k_{i=1}(P^\pi)^{i-1}\epsilon_{k-i+1}.\]
Under the condition of Lemma~\ref{lem:err},$\infn{T^k-\T^\pi(Q^{k})} =\infn{\epsilon_k}\le \varepsilon/(2n)$ for all $k$. This implies 
\begin{equation}\label{Eq:q}
\sp{Q^k-\frac{1}{k+1}\sum^k_{i=0}(\T^\pi)^i Q^0}\le  \frac{\varepsilon}{2}.
\end{equation}

Now we are ready to prove Proposition~\ref{prop:SAVIPE}. 

\begin{proof}
With coefficients in Proposition~\ref{prop:SAVIPE}, by \eqref{Eq:q}, Lemma~\ref{lem:err} and ~\ref{lem:sub}, and union bound over $0 \le k \le K-1$ ,  with at least $1-\delta$ probability,
\begin{align*}
    \sp{\hat{Q}^{\pi_k}-Q_\cR^\pi} &\le \sp{\hat{Q}^{\pi_k}-\frac{1}{k+1}\sum^k_{i=0}(\T^\pi)^i Q^0}+ \sp{Q_\cR^\pi-\frac{1}{k+1}\sum^k_{i=0}(\T^\pi)^i Q^0}
    \\ &\le \frac{\varepsilon}{2}+\frac{\varepsilon}{2}
    \\ &\le \varepsilon
\end{align*}
for all $0 \le k \le K-1$.
\end{proof}

Additionally, following result will be needed for computing sample complexity in next subsection.
 
 We follow the proof framework of \citep{lee2025near}. We keep the notation for all the sequences
generated by SAVIC and
we assume throughout that
\begin{equation}\label{Eq:Ek}
\infn{T^k-\T^\pi(Q^{k})} =\infn{\epsilon_k}\le \varepsilon\qquad \text{for all} \,\, k=0,\ldots,n-1,
\end{equation}
which, in view of Lemma~\ref{lem:err}, holds with probability at least $(1-\delta)$.
Also, we have 
\begin{equation}\label{Eq:spk}\sp{T^k-T^{k-1}}=\sp{D^k} \le \sp{d^k}=\sp{h^k-h^{k-1}}\le \sp{Q^{k}-Q^{k-1}}.
\end{equation}

We first establish two preliminary technical lemmas. 
\begin{lemma}\label{Le:uno}Assuming \eqref{Eq:Ek}, we have  $$\sp{Q^k-Q^\pi_\cR}\leq \sp{Q^0-Q^\pi_\cR}+ k \varepsilon \quad \text{for all} \,\, k=0,\ldots,n.$$ 
\end{lemma}
\begin{proof} From the iteration $Q^{k}=(1-\beta_k)Q^0+\beta_k\,T^{k-1}$ with $\beta_k=\frac{k}{k+1}$ we get
\begin{align}\label{eq:jjj}
    \sp{Q^{k}-Q^\pi_\cR}&\le \mbox{$\frac{1}{k+1}\sp{Q^0-Q^\pi_\cR}+\frac{k}{k+1}\sp{T^{k-1}-Q^\pi_\cR}$}.
\end{align}
Using the invariance of $\sp{\cdot}$ by addition of constants and the nonexpansivity of $\T^\pi$ for this seminorm, a triangle inequality together with $\sp{\cdot}\leq 2\|\cdot\|_\infty$ and the bound \eqref{Eq:Ek} imply
\begin{equation*}\label{Eq:cit}
    \sp{T^{k-1}-Q^\pi_\cR}=\sp{T^{k-1}-\T^\pi(Q^\pi_\cR)}\le 2\varepsilon+\sp{Q^{k-1}-Q^\pi_\cR}
\end{equation*}
which plugged back into \eqref{eq:jjj} yields
$$
\sp{Q^{k}-Q^\pi_\cR}\le \mbox{$\frac{1}{k+1}\sp{Q^0-Q^\pi_\cR}+\frac{2}{k+1}k\varepsilon+\frac{k}{k+1}\sp{Q^{k-1}-Q^\pi_\cR}$}.$$
By induction, the colusion follows.
\end{proof}

\vspace{2ex}

\begin{lemma}\label{Le:dos} Assume \eqref{Eq:Ek} and let $\rho_k=2\sp{Q^0\!-Q^\pi_\cR}+k\varepsilon$. Then 
$$\mbox{$\sp{Q^{k}\!-Q^{k-1}} \le \frac{1}{k}\sum_{i=1}^{k}\frac{\rho_{i+1}}{i+1}\quad \text{for all} \,\, k=1,\ldots,n.$}$$ 
\end{lemma}
\begin{proof} From 
$Q^{k}=\frac{1}{k+1}Q^0+\frac{k}{k+1}T^{k-1}$ and
$Q^{k-1}=\frac{1}{k}Q^0+\frac{k-1}{k}T^{k-2}$
we derive
\begin{equation*} \label{eq:aux}
    Q^{k}-Q^{k-1}
    =\mbox{$ \frac{1}{k(k+1)}(T^{k-1}-Q^0) + \frac{k-1}{k}(T^{k-1}-T^{k-2})$}.
\end{equation*}
Now, from previous proof of Lemma~\ref{Le:uno}, we have $\sp{T^{k-1}-Q^0}\leq \sp{T^{k-1}-Q^\pi_\cR} +\sp{Q^\pi_\cR-Q^0}\leq\rho_{k+1}$,  
and
\begin{align*}
    \sp{Q^{k}-Q^{k-1}} 
    &\le \mbox{$\frac{1}{k(k+1)}\rho_{k+1} +  \frac{k-1}{k}\sp{Q^{k-1}- Q^{k-2}}$}.
\end{align*}
Conclusion follows from induction.
\end{proof}

By Lemma~\ref{Le:dos},  we have
\begin{align*}
    \mbox{$\sp{d^k}$}&\leq\mbox{$\frac{1}{k}\sum_{i=1}^{k}\frac{\rho_{i+1}}{i+1}$}\\
    &\leq \mbox{$\frac{2\ln(k+1)}{k}\sp{Q^0-Q^\pi_\cR}+\varepsilon$}.
\end{align*}

\subsection{Proof of Theorem~\ref{sam:SCPMA}}


Following the framework of \citep{lee2025multi}, consider \emph{inexact policy mirror ascent}
\begin{align*}
\pi_{k+1}(\cdot \,|\, s)
=
\argmax_{p \in \cM_\alpha(\mathcal{A})}
\left\{
b_k \sum_{a\in \cA} \hat{Q}^{\pi_k}(s,a) p(a) 
- D(p(\cdot),\pi_k(\cdot\,|\, s))
\right\},
\qquad \forall s \in \mathcal{S},
\end{align*}
where $\hat{Q}^{\pi_k}$ is an inexact evaluation of $Q^{\pi_k}$. We analyze the convergence of inexact policy mirror ascent under following assumption.
\begin{assump}\label{assum:inexact}
The inexact evaluations $\hat{Q}^{\pi_k}$ satisfy
\begin{align*}
\sp{\hat{Q}^{\pi_k} - Q^{\pi_k}}
\le \varepsilon,
\qquad \forall  k \in \NN.
\end{align*}
\end{assump}
For $\rho \in \cM(\cS)$, define the weighted divergence
\[
D_\rho(\pi,\pi') \;=\; \sum_{s\in\cS}\rho(s)\,D(\pi(\cdot\,|\,s),\pi'(\cdot\,|\,s)).
\] 
Then, following fact holds. We briefly note that although original assumption of Fact~\ref{fact:inexact} requires $\|\cdot\|_\infty$ error bound,  $c\mathbf{1}+\hat{Q}^\pi$ and  $\hat{Q}^\pi$ generate same $\pi_k$ in policy mirror ascent framework. 
\begin{fact}[Appendix D.3, \citep{lee2025multi}]\label{fact:inexact}
Consider a weakly communicating MDP. Under Assumption~\ref{assum:inexact}, for $\pi_0 \in \Pi_\alpha$, the inexact $\alpha$-clipped policy mirror ascent with constant step size $b>0$ generates a sequence of policies $\{\pi_{k}\}^\infty_{k=1}$ satisfying
\[\infn{J^{\pi_\alpha}-J^{\pi_k}} 
\;\le\; \frac{1}{k+1}\!\left( \frac{D_{d^{\pi_\alpha}}({\pi_\alpha}, \pi_0)}{b}
+ C_
\alpha \infn{J^{\pi_\alpha}-J^{\pi_0}} \right)+(2C_\alpha+k+2) \varepsilon.\]
and with  adaptive stepsize $b_{k+1}(C_\alpha - 1) \ge b_k C_\alpha>0$
generates a sequence of policies $\{\pi_{k}\}^\infty_{k=1}$ satisfying
\[\infn{J^{\pi_\alpha}-J^{\pi_k}}
\le
\left(1 - \frac{1}{C_\alpha}\right)^k
\left(
\infn{J^{\pi_\alpha}-J^{\pi_0}}
+
\frac{1}{b_0(C_\alpha - 1)} D_{d^{\pi_\alpha}}({\pi_\alpha}, \pi_0)
\right)
+
4C_\alpha \varepsilon.\]
\end{fact}

Now, we prove  Theorem~\ref{sam:SCPMA}

\begin{proof}
By Fact~\ref{fact:per},
   there exist $\pi \in \Pi_{\alpha}$ such that $\infn{\pi^\star-\pi} \le \alpha$, and this implies $J^{\star} - J^{\pi_\alpha}
\le \frac{\varepsilon}{2}$. 

First, for constant step size, with coefficients in Theorem~\ref{sam:SCPMA},  Assumption~\ref{assum:inexact} is satisfied with $ \frac{\varepsilon}{4(2C_\alpha+K+2)}$. Then, by Proposition~\ref{prop:SAVIPE},    $\mbox{$\sp{d^k}$}\le  \mbox{$\frac{2\ln(k+1)}{k}\sp{Q^0-Q^\pi_\cR}+\varepsilon$}$,  and Fact~\ref{fact:inexact} with sample complexity 
 \begin{align*}
   & \sum^{K-1}_{k=0}\sum^{n_k-1}_{i=0}{t_i^{\pi_k}}\\&=\sum^{K-1}_{k=0}\widetilde{O}\left((2C_\alpha+K+2)^2t_{\mathrm{cov}}t^{\pi_k}_{\mathrm{tar}}\sp{Q^{\pi_k}}^2/\varepsilon^2+ 1/d^{\pi_b}_{\mathrm{min}} (t^{\pi_k}_{\mathrm{tar}})^2\sp{Q^{\pi_k}}^4(2C_\alpha+K+2)^4/\varepsilon^{4}\right)
\\&\le\widetilde{O}\left(gt_{\mathrm{cov}}t_{\alpha}C_\alpha^2Q_\alpha^3/\varepsilon^2+(t_{\mathrm{cov}}t_{\alpha}g^3Q_\alpha^2+1/d^{\pi_b}_{\mathrm{min}}g C^4_\alpha t_{\alpha}^2 Q_\alpha^4)/\varepsilon^5+ 1/d^{\pi_b}_{\mathrm{min}} t_{\alpha}^4Q_\alpha^4g^5/\varepsilon^{9}\right).
\end{align*} 
Lastly, by adding $n_{\mathrm{hit}}+n_{\mathrm{rec}}$, we have total sample complexity.

For adaptive step size $\eta_{k+1}(C_\alpha  - 1) \ge b_k C_\alpha $, with coefficients in Theorem~\ref{sam:SCPMA}, Assumption~\ref{assum:inexact} is satisfied with $ \frac{\varepsilon}{16C_\alpha}$. Then, by Proposition~\ref{prop:SAVIPE}, $\mbox{$\sp{d^k}$}\le  \mbox{$\frac{2\ln(k+1)}{k}\sp{Q^0-Q^\pi_\cR}+\varepsilon$}$,  and Fact~\ref{fact:inexact}, with at least $1-\delta$ probability, $ J^{\pi_\alpha}-J^{\pi_k}_\mu \le \frac{\varepsilon}{2}$ with sample complexity  
 \begin{align*}
   & n_{\mathrm{hit}}+n_{\mathrm{rec}}+\sum^{K-1}_{k=0}\sum^{n_k-1}_{i=0}{t_i^{\pi_k}}\\&=\widetilde{O}\left(t_{\mathrm{hit}}+K(t_\alpha t_{\mathrm{cov}} Q_\alpha^2C_\alpha^2/\varepsilon^{2}+ 1/d^{\pi_b}_{\mathrm{min}} t_\alpha^2Q_\alpha^4C_\alpha^4/\varepsilon^{4})\right).
\end{align*}  
 Lastly, by combining results, we have $J^\star-J^{\pi_{k}}\le \varepsilon$.

\end{proof}

\newpage
\input{checklist.tex}

\end{document}

%% file: checklist.tex
\section*{NeurIPS Paper Checklist}

\begin{enumerate}

\item {\bf Claims}
    \item[] Question: Do the main claims made in the abstract and introduction accurately reflect the paper's contributions and scope?
    \item[] Answer: \answerYes{} 
    \item[] Justification: Yes main claims made in the abstract and introduction accurately reflect the paper's contributions and scope.  

\item {\bf Limitations}
    \item[] Question: Does the paper discuss the limitations of the work performed by the authors?
    \item[] Answer: \answerYes{} 
    \item[] Justification: Yes, authors  discuss the limitations of the work in conclusion section.

\item {\bf Theory assumptions and proofs}
    \item[] Question: For each theoretical result, does the paper provide the full set of assumptions and a complete (and correct) proof?
    \item[] Answer: \answerYes{} 
    \item[] Justification: Yes,  the paper provide the full set of assumptions and a complete (and correct) proof.

    \item {\bf Experimental result reproducibility}
    \item[] Question: Does the paper fully disclose all the information needed to reproduce the main experimental results of the paper to the extent that it affects the main claims and/or conclusions of the paper (regardless of whether the code and data are provided or not)?
    \item[] Answer: \answerNA{} 
    \item[] Justification: This paper does not include the  experimental results.

\item {\bf Open access to data and code}
    \item[] Question: Does the paper provide open access to the data and code, with sufficient instructions to faithfully reproduce the main experimental results, as described in supplemental material?
    \item[] Answer: \answerNA{} 
    \item[] Justification: This paper does not include any experimental result.

\item {\bf Experimental setting/details}
    \item[] Question: Does the paper specify all the training and test details (e.g., data splits, hyperparameters, how they were chosen, type of optimizer) necessary to understand the results?
    \item[] Answer:  \answerNA{} 
    \item[] Justification: This paper does not include any experimental result.

\item {\bf Experiment statistical significance}
    \item[] Question: Does the paper report error bars suitably and correctly defined or other appropriate information about the statistical significance of the experiments?
    \item[] Answer: \answerNA{} 
    \item[] Justification: This paper does not include any experiment.

\item {\bf Experiments compute resources}
    \item[] Question: For each experiment, does the paper provide sufficient information on the computer resources (type of compute workers, memory, time of execution) needed to reproduce the experiments?
    \item[] Answer:  \answerNA{} 
    \item[] Justification: This paper does not include any expereiments.
    
\item {\bf Code of ethics}
    \item[] Question: Does the research conducted in the paper conform, in every respect, with the NeurIPS Code of Ethics \url{https://neurips.cc/public/EthicsGuidelines}?
    \item[] Answer: \answerYes{} 
    \item[] Justification: This research conducted in the paper conform, in every respect, with the NeurIPS Code of Ethics.

\item {\bf Broader impacts}
    \item[] Question: Does the paper discuss both potential positive societal impacts and negative societal impacts of the work performed?
    \item[] Answer: \answerNA{} 
    \item[] Justification: This paper is theory paper and has no potential negative or positive societal impacts. 
    
\item {\bf Safeguards}
    \item[] Question: Does the paper describe safeguards that have been put in place for responsible release of data or models that have a high risk for misuse (e.g., pre-trained language models, image generators, or scraped datasets)?
    \item[] Answer: \answerNA{} 
    \item[] Justification: This paper does not include any experiments.

\item {\bf Licenses for existing assets}
    \item[] Question: Are the creators or original owners of assets (e.g., code, data, models), used in the paper, properly credited and are the license and terms of use explicitly mentioned and properly respected?
    \item[] Answer: \answerNA{}
    \item[] Justification: This paper does not include any experiments.

\item {\bf New assets}
    \item[] Question: Are new assets introduced in the paper well documented and is the documentation provided alongside the assets?
    \item[] Answer: \answerNA{}
    \item[] Justification: This paper does not include any experiments.

\item {\bf Crowdsourcing and research with human subjects}
    \item[] Question: For crowdsourcing experiments and research with human subjects, does the paper include the full text of instructions given to participants and screenshots, if applicable, as well as details about compensation (if any)? 
    \item[] Answer: \answerNA{} 
    \item[] Justification: This paper does not include any experiments.

\item {\bf Institutional review board (IRB) approvals or equivalent for research with human subjects}
    \item[] Question: Does the paper describe potential risks incurred by study participants, whether such risks were disclosed to the subjects, and whether Institutional Review Board (IRB) approvals (or an equivalent approval/review based on the requirements of your country or institution) were obtained?
    \item[] Answer: \answerNA{} 
    \item[] Justification: There is no potential risks incurred by study participants, whether
such risks were disclosed to the subjects, and whether Institutional Review Board (IRB)
approvals 
with human subjects.

\item {\bf Declaration of LLM usage}
    \item[] Question: Does the paper describe the usage of LLMs if it is an important, original, or non-standard component of the core methods in this research? Note that if the LLM is used only for writing, editing, or formatting purposes and does \emph{not} impact the core methodology, scientific rigor, or originality of the research, declaration is not required.
    \item[] Answer: \answerNA{} 
    \item[] Justification: We use LLM to fix grammar error.

\end{enumerate}